\documentclass[journal]{IEEEtaes}
\usepackage{etoolbox}

\newtoggle{ONECOLUMN}
\togglefalse{ONECOLUMN}

\usepackage[colorlinks,urlcolor=blue,linkcolor=blue,citecolor=blue]{hyperref}

\usepackage{color,array,amsthm}

\usepackage{graphicx}

\iftoggle{ONECOLUMN}{
}{
    \jvol{XX}
    \jnum{XX}
    \jmonth{XXXXX}
    \paper{1234567}
    \pubyear{2020}
    \doiinfo{TAES.2020.Doi Number}
}

\newtheorem{theorem}{Theorem}
\newtheorem{lemma}{Lemma}
\usepackage[T1]{fontenc}
\usepackage[latin9]{inputenc}
\usepackage{amsmath}
\usepackage{amssymb}
\usepackage{caption}
\usepackage{subcaption}
\usepackage{algorithm,algpseudocode,tabularx}

\theoremstyle{definition}

\usepackage[    
    natbib=true,
    style=numeric,
    sorting=none]{biblatex} 

\newcommand{\SE}[1]{\ensuremath{\mathrm{SE}\mathopen{}\left(#1\right)\mathclose{}}}
\newcommand{\SO}[1]{\ensuremath{\mathrm{SO}\mathopen{}\left(#1\right)\mathclose{}}}
\newcommand{\se}[1]{\ensuremath{\mathfrak{se}\mathopen{}\left(#1\right)\mathclose{}}}

\newcommand{\ssm}[1]{\ensuremath{\left[#1\right]_\times}}
\newcommand{\CS}[1]{\ensuremath{\mathbf{R}_{#1}}}
\newcommand{\G}[1]{{\ensuremath{G_{#1}}}}

\newcommand{\defeq}{\doteq}
\newcommand{\ExpI}[2]{\ensuremath{\mathrm{Exp}_{I}^{#1}\mathopen{}\left( #2 \right)\mathclose{}}}
\newcommand{\LogI}[2]{\ensuremath{\mathrm{Log}_{I}^{#1}\mathopen{}\left( #2 \right)\mathclose{}}}
\newcommand{\jr}[2]{\ensuremath{\mathrm{J}_{r}^{#1}\mathopen{}\left( #2 \right)\mathclose{}}}
\newcommand{\jl}[2]{\ensuremath{\mathrm{J}_{l}^{#1}\mathopen{}\left( #2 \right)\mathclose{}}}
\newcommand{\Ad}[2]{\ensuremath{\mathrm{Ad}_{#2}^{#1} }}
\newcommand{\ad}[2]{\ensuremath{\mathrm{ad}_{#2}^{#1} }}
\newcommand{\kt}{ \Delta }
\newcommand{\gauss}[2]{\ensuremath{\mathcal{N}\left({#1},{#2}\right)}}
\newcommand{\xtkkm}{\tilde{x}_{k \mid  k^{-}}}
\newcommand{\xtkmkm}{\tilde{x}_{k^{-} \mid  k^{-}}}
\newcommand{\Pkkm}{P_{k \mid  k^{-}}}
\newcommand{\Pkmkm}{P_{k^{-} \mid  k^{-}}}

\DeclareBibliographyCategory{asterisk}
\makeatletter
\renewcommand*{\p@subsection}{\thesection.}
\makeatother

\begin{document}
\title{The Integrated Probabilistic Data Association Filter Adapted to Lie Groups}

\iftoggle{ONECOLUMN}{
    \author{Mark E. Petersen, Randal W. Beard}
}{
    \author{Mark E. Petersen}
    \affil{Brigham Young University, Provo, UT} 

    \author{Randal W. Beard}
    \member{Fellow, IEEE}
    \affil{Brigham Young University, Provo, UT} 

    \receiveddate{This work has been funded by the Center for Unmanned Aircraft Systems (C-UAS), a National Science Foundation Industry/University Cooperative Research Center (I/UCRC) under NSF award No. IIP-1650547, along with significant contributions from C-UAS industry members. }
    \corresp{{\itshape (Corresponding author: R. Beard, email: beard@byu.edu)}. 
    }
}

\maketitle

\begin{abstract}
The Integrated Probabilistic Data Association Filter (IPDAF) is a target tracking algorithm based on the Probabilistic Data Association Filter that calculates a statistical measure that indicates if an estimated representation of the target properly represents the target or is generated from non-target-originated measurements. The main contribution of this paper is to adapt the IPDAF to constant velocity target models that evolve on connected, unimodular Lie groups, and where the measurements are also defined on a Lie group. We present an example where the methods developed in the paper are applied to the problem of tracking a ground vehicle on the special Euclidean group \SE{2}.   
\end{abstract}


\section{Introduction}
When tracking a single target whose initial position is unknown, all real-world sensors produce non-target-originated measurements (e.g. false measurements or clutter). In the presence of dense clutter, it is a challenge to locate and track the target since it is difficult to distinguish between target-originated measurements and false measurements. The typical approach is to use new measurements to either improve the estimate of an existing track (a track is a representation of the target which consists of at least the state estimate) or initialize new tracks. If the clutter density is high, numerous tracks that do not represent the target are initialized from clutter. 

An example scenario illustrating the challenge of tracking in clutter is depicted in Fig.~\ref{fig:measurement_clutter}, where the black dots represent measurements, the green car represents the true target, and the blue cars represent the tracks currently in memory. The left image represents a time step when measurements are received for the first time. The right image represents a subsequent time step when previous measurements are used to initialize new tracks and additional measurements are received. The challenge is to identify which track represents the target.

\begin{figure}[thb]
\centering
    \includegraphics[width=0.9\linewidth]{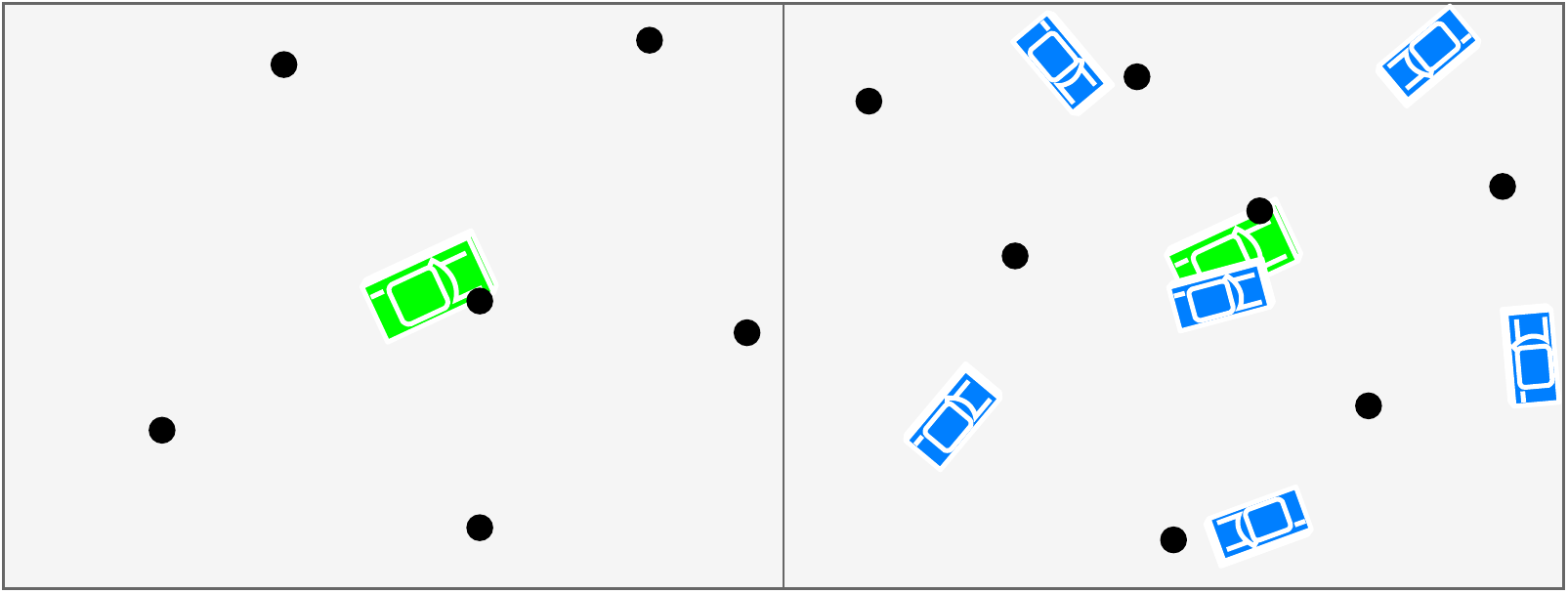}
    \caption{A depiction of the challenge of identifying which track represents the target. The black dots represent measurements, the green car represents the target, and the blue cars represent tracks. The left image represents the first time step that the measurements are received. The right image represents the next time step when previous measurements are used to initialize new tracks and additional measurements are received.}
     \label{fig:measurement_clutter}
\end{figure}

Identifying which track best represents the target requires an additional estimate called the {\em track likelihood}. Tracks with a low track likelihood are rejected and pruned from memory, while tracks with a high track likelihood are confirmed as good tracks. The confirmed track with the highest track likelihood can be used as the best estimate of the target. 

Different approaches to calculating the track likelihood depend on the data association algorithm. 
Data association is the process of assigning new measurements to existing tracks so that the associated measurement can be used to improve the estimates of the tracks. There are two types of data association: hard data association and soft data association. Hard data association assigns at most one new measurement to each track, and soft data association can assign multiple new measurements to each track. 

Tracking algorithms that use hard data association, such as the Nearest Neighbor filter (NNF) \cite{Bar-Shalom1988}, the Global Nearest Neighbor filter (GNNF) \cite{Bhatia2010,Konstantinova2006} and track splitting algorithms, commonly use the likelihood function or the negative log likelihood function (NLLF) to determine if a track should be rejected \cite{Bar-Shalom2011}. The likelihood function measures the joint probability density of all of the measurements associated with a track. If the track's likelihood function falls below some threshold, the track is rejected. This approach does not indicate which track best represents the target, only which tracks should be removed.  However, it is common to suppose that the track with the highest track likelihood best represents the target. 

Another approach to quantify the track likelihood when using hard data association is based on the sequential probability ratio test (SPRT) \cite{Wald1947}. The SPRT uses a sequence of data to either confirm a null hypothesis or reject the alternate hypothesis by analyzing the probability ratio of the two hypotheses. In terms of tracking, the SPRT calculates the joint probability density of the measurements associated to a track for the hypothesis that all the measurements are target originated and for the hypothesis that all the measurements are false. It then takes the ratio of the two probability densities and either rejects the track if the ratio is below a threshold, confirms the track if the ratio is above a threshold or continues to gather more information as new measurements are received until the track can be confirmed or rejected \cite{Li2002a}. The SPRT is used in a variety of tracking algorithms including the Multiple Hypothesis Tracker (MHT) \cite{Blackman2004}.

Many of the common hard data association algorithms are either computationally inexpensive and not robust to clutter like the NNF and the GNNF or very computationally expensive and robust like the MHT and other track splitting methods. On the other hand, soft data associations algorithms, like the probabilistic data association filter (PDAF) \cite{Bar-Shalom1975,Bar-Shalom2009}, offer a good balance between robustness and computational expense. Similar to a track splitting method, when multiple measurements are associated to a track the PDAF makes a copy of the track for every associated measurement, and then the copy is updated with one of the associated measurements resulting in the track splitting for every measurement. The PDAF differs from the track splitting method in that after track splitting, the split tracks are fused together into a single track according to the likelihood of each split track representing the target. 

A nice feature of the PDAF is that it can be used with many different types of dynamic models including dynamic models evolving on Lie groups. Using Lie groups to model rigid body dynamics is a recent approach~\cite{Mangelson2019,Barfoot2014,Barfoot2019,Sola2018,Park1995,LeeT2010}. Lie groups are beneficial since they naturally model the target's pose and motion without suffering from the singularities inherent in other attitude representations like Euler angles, axis-angle, etc. In addition, Lie group models more realistically model the uncertainty that exists with physical systems. Because of these features, Lie groups have shown increased accuracy in estimation \cite{Long2013}. Many physical systems are more naturally represented using Lie groups including satellite attitude dynamics, fixed-wing unmanned aircraft systems, multirotors and other flying objects, cars on road networks, ground robots and walking pedestrians.

The two aforementioned methods of quantifying the track likelihood do not work for soft data association algorithms. For this reason, the PDAF was extended in \cite{Musicki1994} to calculate the track likelihood using a novel approach called the Integrated Probabilistic Data Association Filter (IPDAF).
To our knowledge, the IPDAF has not been adapted to Lie groups.  However, the joint integrated probabilistic data association filter (JIPDAF) was adapted to the Lie group $\SE{2}$ in~\cite{Cesic2016a}. The JIPDAF is the adaptation of the IPDAF to tracking multiple targets. When tracking only a single target, the JIPDAF reduces to the IPDAF. However, the reduction would require understanding the more complicated JIPDAF instead of the simpler IPDAF. Since the main focus of \cite{Cesic2016a} was not the adaptation of the JIPDAF to general Lie groups, the algorithm was not derived and explained in detail.
 
The purpose of this paper is to present the IPDAF adapted to connected, unimodular Lie groups in a tutorial manner by providing enough detail to make it clear how it can be implemented.  We refer to the resulting algorithm as the Lie group integrated probabilistic data association filter (LG-IPDAF).  Our focus in this paper will be on nearly constant velocity models, where the velocity is an element of the Lie algebra.  We will primarily consider measurement models where the measurement is an element of the Lie group, e.g., position and orientation for $\SE{2}$.  In addition, we make several unique contributions by explaining the validation region for Lie groups, showing how the indirect Kalman filter is used with Lie groups and by presenting the system model more generally than~\cite{Cesic2016a} by representing the target's pose as an element of a Lie group and its velocity as an element of the associated Lie algebra instead of representing the velocity using a constrained element of the Lie group.  The modeling approach used in~\cite{Cesic2016a} can be problematic when angular velocities are sufficiently high since the mapping from the Lie algebra to the Lie group is surjective and not bijective.  Our last contribution is to present the material generically so that it can be easily applied to any Lie group, whereas in~\cite{Cesic2016a} the algorithm is specifically presented for $\SE{2}$.

The rest of this paper is outlined as follows. In Section~\ref{sec:Lie-Theory-Review} we review basic concepts of Lie group theory to establish notation and to enable those not familiar with Lie groups to follow the subsequent development. In Section~\ref{sec:lgipdaf_overview} we present an overview of the LG-IPDAF algorithm. In Section~\ref{sec:lgipdaf_system_model} we present the system model used by the LG-IPDAF. In Sections~\ref{sec:lgipdaf} we give detailed derivations of the key elements of the LG-IPDA Filter including the prediction, data association, and update steps as well as the track initialization scheme. In Section~\ref{sec:lgipdaf_examples} we present a simple simulation example and conclude in Section~\ref{sec:lgipdaf_conclusion}.  Detailed proofs of some of the results are given in the Appendices to facilitate the flow of the paper.

\section{Lie Group Theory Review\label{sec:Lie-Theory-Review}}

Lie group theory is an extensive topic that we cannot completely cover in this
paper, but excellent tutorials for robotic applications are given in~\cite{Sola2018,Hertzberg2013}. We also recommend  \cite{Barfoot2019,Stillwell2008,Hall2003,Bullo2005a,LeeJohn2013,Abraham1998} for a more rigorous treatment of Lie group theory.
In this paper, we restrict our discussion to connected, unimodular Lie groups, an assumption that is required by our model of the uncertainty on the Lie group as will be explained later. 
The objective of this section is to provide a brief review and to establish
notation. Even though we illustrate our notation using the Lie group $\SE{2}$, the provided definitions are sufficiently generic to apply to every connected, unimodular Lie group. 

\subsection{Lie Group and Lie Algebra}

Let $G$ denote a Lie group and $\mathfrak{g}$ denote its corresponding
Lie algebra. We identify the Lie algebra with the tangent space of $G$
at the identity element. For example, the special Euclidean group \SE{2} is a matrix Lie group used to model rigid body motion in two dimensions. It is isomorphic to the set 
\begin{equation}
    \SE{2} \cong \left\{ \left.
    \begin{bmatrix}
    R & p \\
    0_{1\times2} & 1
    \end{bmatrix}
    \right| R\in \SO{2},\, p\in \mathbb{R}^2
    \right\},
\end{equation}
with the group operator being matrix multiplication, $R$ denoting a 2-dimensional rotation matrix that represents the attitude of the rigid body, and $p$ denoting the position of the rigid body. In this work we will use $R$ to denote the rotation from the body frame to the inertial frame, and $p$ to denote the position of the body with respect to the inertial frame expressed in the inertial frame. 

The Lie algebra of \SE{2} is denoted \se{2} and is isomorphic to the set
\begin{equation}
    \mathfrak{se}\left(2\right) \cong \left\{ \left.
    \begin{bmatrix}
    \left[\omega\right]_\times & \rho \\
    0_{1\times2} & 0
    \end{bmatrix}
    \right| \omega \in \mathbb{R},\, \rho\in \mathbb{R}^2
    \right\},
\end{equation}
with $\rho$ denoting the translational velocity, $\omega$ denoting the angular velocity, and $\left[\cdot \right]_\times$ being the skew symmetric operator defined as 
\begin{equation}
    \left[ \omega \right]_\times = \begin{bmatrix}
    0 & -\omega \\
    \omega & 0
    \end{bmatrix}.
\end{equation}
In this work we will use $\rho$ to denote the translational velocity of the body with respect to the inertial frame expressed in the body frame, and $\omega$ to denote the rotational velocity of the body with respect to the inertial frame expressed in the body frame. 

The Lie algebra can take on various representations.  However, by taking
advantage of its algebraic structure, elements of the Lie algebra can be expressed as the
linear combination of orthonormal basis elements $\left\{\mathbf{e}_{i}\right\}\subset \mathfrak{g}$. For example, let $\mathbf{v}\in\se{2}$,
then $\mathbf{v}=\sum_{i=1}^{3}a_{i}\mathbf{e}_{i}$ with $a_{i}$
denoting the coefficient associated with $\mathbf{e}_{i}$ and where 
\begin{equation*}
    \mathbf{e}_1 = \begin{bmatrix}
0 & -1 & 0 \\
1& 0 & 0 \\
0 & 0 & 0
\end{bmatrix}  \hspace{0.5em} \mathbf{e}_2 = \begin{bmatrix}
0 & 0 & 1 \\
0& 0 & 0 \\
0 & 0 & 0
\end{bmatrix}  \hspace{0.5em} \mathbf{e}_3 = \begin{bmatrix}
0 & 0 & 0 \\
0& 0 & 1 \\
0 & 0 & 0
\end{bmatrix}.
\end{equation*}
The coefficients form an algebraic space isomorphic to the Lie algebra that we will refer to as the Cartesian algebraic space denoted \CS{G}, where the subscript indicates the corresponding Lie group. Elements in the Cartesian algebraic space can be represented using matrix notation as $v=[a_1,a_2,\ldots]^\top$. We distinguish elements of the Lie algebra from elements of the Cartesian algebraic space using bold font notation for elements of the Lie algebra.

The wedge, $\cdot^\wedge$, and vee, $\cdot^\vee$, functions are linear isomorphisms used to map between the Lie algebra and the Cartesian algebraic space. We denote these functions respectively as
\begin{subequations} 
\begin{align*}
\cdot^{\wedge}:\CS{G}\to\mathfrak{g}; & \quad \left( v\right) \mapsto \mathbf{v}\\
\cdot^{\vee}:\mathfrak{g}\to\CS{G}; & \quad \left( \mathbf{v} \right) \mapsto v.
\end{align*}
\end{subequations}
For \se{2} the vee map is defined as 
\begin{equation}
    \left(\begin{bmatrix}
    \left[\omega\right]_\times & \rho \\
    0_{1\times2} & 0
    \end{bmatrix}\right)^\vee = 
\begin{bmatrix}
\rho \\
\omega
\end{bmatrix}=v,
\end{equation}
and the wedge map is the inverse.

\subsection{Exponential Map}
For Riemannian manifolds, a geodesic is the shortest path between two points. The exponential function on the Lie group $G$, denoted $\text{Exp}^G:\, G\times \CS{G} \to G$, is a geodesic that starts at a point $g_{1}\in G$ and travels in the direction of a tangent vector $v\in\CS{G}$ for unit time to the point $g_{2}\in G$ as stated in Proposition~2.7 of \cite{Carmo1992}. We denote the exponential function and its inverse, $\text{Log}^G$, as
\begin{subequations} 
\begin{align*}
\text{Exp}^G:\, & G\times \CS{G}  \to G; \, \left(g_1,v\right)\mapsto g_2\\
\text{Log}^G:\, & G \times G  \to\CS{G}; \, \left(g_2,g_1\right)\mapsto v.
\end{align*}
\end{subequations} 

When working with Lie groups, it is common to restrict the definition of the exponential map to the identity element of the group; we denote this restriction as 
\begin{subequations} 
\begin{align*}
\text{Exp}^G_I:\, & \CS{G}  \to G; \, \left(v\right)\mapsto g_3\\
\text{Log}^G_I:\, &  G  \to\CS{G}; \, \left(g_3\right)\mapsto v.
\end{align*}
\end{subequations} 
Lie groups allow the restricted exponential map to be moved to another element of the group by applying the left or right group action. Using left-trivialization, in other words the left group action, we define the relation
\begin{subequations}\label{eq:exp_relation}
\begin{align}
    \text{Exp}^G\left(g_1,v \right) &= g_1\bullet \ExpI{G}{v} = g_2 \\
    \text{Log}^G \left(g_2,g_1\right) &= \LogI{G}{ g_1^{-1}\bullet g_2} = v
\end{align}
\end{subequations}
with $\bullet$ denoting the group operator that we omit in the future. A depiction of the exponential map is given in Fig.~\ref{fig:geodesic_def} where the sphere represents the manifold, the plane represents the tangent space that extends to infinity, the arrow in the tangent space represents $v$, and the arrow on the manifold represents the geodesic from $g_1$ to $g_2$ in the direction of $v$. 
\begin{figure}[thb]
\centering
    \includegraphics[width=0.7\linewidth]{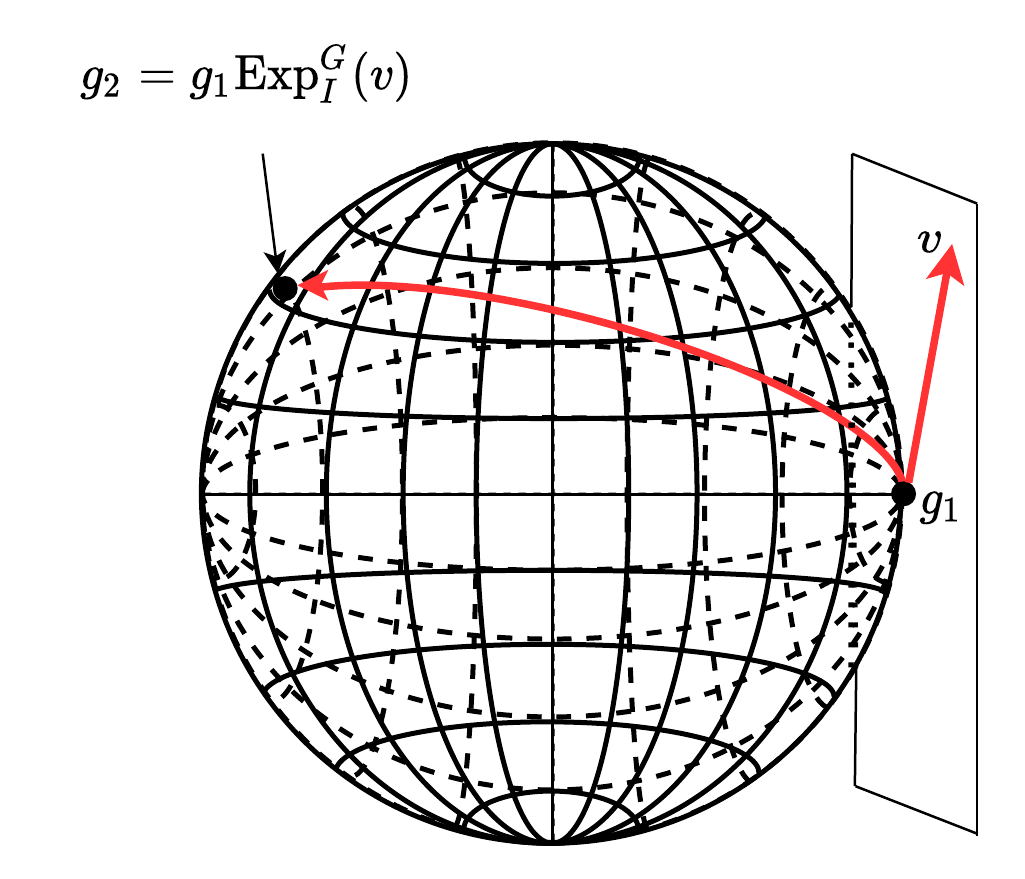}
    \caption{A depiction of a geodesic starting at $g_1$, moving in the direction of $v$ and ending at $g_2$.}
     \label{fig:geodesic_def}
\end{figure}
For matrix Lie groups, the exponential and logarithm maps at the identity element of the group are the matrix exponential and matrix logarithm. These maps operate on the Lie algebra, but their definitions are extended to the Cartesian algebraic space using the wedge and vee operations. Unfortunately, not every Lie group is isomorphic to a matrix Lie group, but fortunately most of Lie groups that appear in robotics and control applications are~\cite{Hall2003}. Thus, the matrix exponential can serve as the exponential map for the majority of the interesting Lie groups.

Let $v = [\rho^\top, \omega]^\top$ and $g = \begin{bmatrix}
R & p \\ 0 & 0
\end{bmatrix}$. For \SE{2} the matrix logarithm and exponential mappings have closed-form expressions given by 
\begin{align*}
\ExpI{\SE{2}}{v} & =
\begin{bmatrix}
\text{expm} \left(\ssm{\omega} \right) & D\left(\omega \right)\rho\\
0_{1 \times 2} & 1
\end{bmatrix}\\
\LogI{\SE{2}}{g} & = 
\begin{bmatrix}
\text{logm}\left(R\right) & D^{-1}\left(\text{logm}\left(R\right) \right)p\\
0_{1 \times 2} & 0
\end{bmatrix},
\end{align*}
where
\begin{align}
\text{expm}\left(\ssm{\omega}\right) & = \begin{bmatrix}
\cos{\omega} & -\sin{\omega} \\
\sin{\omega} & \cos{\omega}
\end{bmatrix} \\
\text{logm}\left(R\right) &= \text{arctan2}\left(R_{21},R_{11} \right) = \theta\\
D \left(\omega\right) & = 
\frac{\sin\omega}{\omega}I+\frac{1-\cos\omega}{\omega}\left[1\right]_{\times} \label{eq:se2_v} \\
D^{-1}\left(\theta\right) & =
\frac{\theta\sin\theta}{2\left(1-\cos\theta\right)}I-\frac{\theta}{2}\left[1\right]_{\times},
\end{align}
and where we note that $D(0)$ and $D^{-1}(0)$ are well defined since $\lim_{\omega\to 0} \sin\omega/\omega = 1$, $\lim_{\omega\to 0} (1-\cos\omega)/\omega = 0$, and $\lim_{\theta\to 0} \theta\sin\theta/(2(1-\cos\theta))=1$.  We also note that when $\omega$ is small, $D(\omega)\approx I + \frac{1}{2}\left[\omega\right]_\times$ and $D^{-1}(\theta)\approx I - \frac{1}{2}\left[\theta\right]_\times$.

\subsection{Adjoint}

The matrix adjoint of $g\in G$ is a representation of $g$ that acts on $\CS{G}$ and is generically defined as 
\[
\text{Ad}^G_{g}:\CS{G}\to\CS{G};\quad\left(v\right)\mapsto \Ad{G}{g}v,
\]
and represents, for example, a change of coordinates from one location on the manifold to another.
A useful property of the adjoint is 
\begin{equation}\label{eq:lie_groups_review_adjoint_relation}
g\ExpI{G}{v}=\ExpI{G}{\Ad{G}{g}v}g.    
\end{equation}
For \SE{2}, the matrix adjoint of $g \in \SE{2}$ is
\begin{equation}
    \text{Ad}^{\SE{2}}_g = \begin{bmatrix}
    R & -\ssm{1} p\\
    0_{1\times 2} &1
    \end{bmatrix}.
\end{equation}

The matrix adjoint of $v_1 \in \CS{G}$ is a representation of $\CS{G}$ that acts on $v_2 \in \CS{G}$ generically defined as 
\begin{equation*}
\text{ad}^G_{v_1}:\CS{G}\to \CS{G}; \quad \left( v_2\right)\mapsto \text{ad}^G_{v_1}v_2.
\end{equation*}
For \se{2}, the matrix representation of the adjoint is
\begin{equation*}
\text{ad}^{\SE{2}}_{v} = \begin{bmatrix}
\ssm{w} & -\ssm{1}\rho \\
0_{1\times 2} & 0
\end{bmatrix}.
\end{equation*}

\subsection{Jacobian of the Matrix Exponential}

When working with Lie groups, we need the differential of the exponential
and logarithm functions. These differentials are commonly called
the right and left Jacobians. The right and left Jacobians and their inverses are defined to map elements of \CS{G} to the general linear group (set of invertible matrices) that acts on \CS{G}. For matrix Lie groups, they are defined as

\begin{align*}
\jr{G}{v} & =\sum_{n=0}^{\infty}\frac{\left(-\ad{G}{v}\right)^{n}}{\left(n+1\right)!}, \hspace{0.5em}  \jl{G}{v}=\sum_{n=0}^{\infty}\frac{\left(\ad{G}{v}\right)^{n}}{\left(n+1\right)!},
\end{align*}

\begin{align*}
\jr{G^{-1}}{v} & =\sum_{n=0}^{\infty}\frac{B_{n}\mathopen{}\left(-\ad{G}{v}\right)^{n}\mathclose{}}{n!},  \hspace{0.5em} \jl{G^{-1}}{v}=\sum_{n=0}^{\infty}\frac{B_{n}\mathopen{}\left(\ad{G}{v}\right)^{n}\mathclose{}}{n!},
\end{align*}
where $B_{n}$ are the Bernoulli numbers, the subscripts $r/l$ indicate the right and left Jacobian, and the superscript indicates the corresponding Lie group. The derivation of the left
and right Jacobians stems from the Baker-Campbell-Hausdorff formula~\cite{Hall2003,Barfoot2019}. The
right Jacobian has the properties that for any $v\in\CS{G}$ and any small $\tilde{v}\in\CS{G}$, 
\begin{subequations}\label{eq:Jr_property}
\begin{align}
\ExpI{G}{v + \tilde{v}} & \approx\ExpI{G}{v}\ExpI{G}{\jr{G}{v}\tilde{v}}\\
\ExpI{G}{v}\ExpI{G}{\tilde{v}} & \approx\ExpI{G}{v+\jr{G^{-1}}{v}\tilde{v}}.
\end{align}
\end{subequations}
Similarly for the left Jacobian,
\begin{align*}
\ExpI{G}{v+\tilde{v}} & \approx\ExpI{G}{\jl{G}{v}\tilde{v}}\ExpI{G}{v}\\
\ExpI{G}{\tilde{v}}\ExpI{G}{v} & \approx\ExpI{G}{v+\jl{G^{-1}}{v}\tilde{v}}.
\end{align*}

For \SE{2}, the Jacobians have closed form solutions, and
the right Jacobian for \SE{2} is 
\begin{equation}
    \jr{\SE{2}}{v} = \begin{bmatrix}
    W_r \left(\omega\right) & D_r \left(\omega\right)\rho \\
    0_{1\times 2} & 1
    \end{bmatrix}
\end{equation}
where 
\begin{align*}
    W_r \left(\omega\right) &=
    \frac{\cos{\omega}-1}{\omega}\left[ 1 \right]_\times + \frac{\sin{\omega}}{\omega}I  \\
    D_r \left(\omega\right) &=
    \frac{1-\cos{\omega}}{\omega^2}\left[ 1 \right]_\times + \frac{\omega - \sin{\omega}}{\omega^2}I,
\end{align*}
and where we again note that $W_r(0)$ and $D_r(0)$ are well defined and that for small $\omega$, 
$W_r(\omega)\approx I - \frac{1}{2}[\omega]_\times$ and $D_r(\omega)\approx \frac{1}{2}[1]_\times$.

\subsection{Direct Product Group}

In this paper, we assume that the target has nearly constant velocity, implying that the target's state is modeled as an element of the Lie group formed from the direct product of $G$ and $\CS{G}$, denoted $x=\left(g,v\right)\in \G{x} \defeq G\times\CS{G}$, where the target's pose is expressed as an element of $G$ and its velocity is expressed as an element of $\CS{G}$. The operator of this Lie group is inherited from its subgroups $G$ and $\CS{G}$. Since $\CS{G}$ is an algebraic space, it has an Abelian group structure with the group operator being addition. Thus, the group operator and inverse of $\G{x}$ are
\begin{align*}
x_{1}\bullet x_{2} &= (g_1, v_1)\bullet(g_2, v_2) = \left(g_{1} g_{2},v_{1}+v_{2}\right)\\
x_{1}^{-1} &= (g_1, v_1)^{-1} = \left(g_{1}^{-1},-v_{1}\right).
\end{align*}

The corresponding Cartesian algebraic space of $\G{x}$
is $\CS{x} \defeq \CS{G}\times\CS{G}$, and its exponential and logarithmic maps are
\begin{subequations}\label{eq:exp_map_direct_product}
\begin{align}
\ExpI{\G{x}}{u} & = \ExpI{\G{x}}{(u^g, u^v)} = \left( \ExpI{G}{u^g}, u^v \right) \\ 
\LogI{\G{x}}{x} & = \LogI{\G{x}}{(g,v)} =\left( \LogI{G}{g}, v \right),
\end{align}
\end{subequations}
where $u=\left(u^g,u^v\right)\in \CS{x}$. The right Jacobian in matrix notation of $\G{x}$ is 
\begin{equation}
    \jr{\G{x}}{u} = \jr{\G{x}}{(u^g, u^v)} = \begin{bmatrix}
    \jr{G}{u^g} & 0 \\
    0 & I
    \end{bmatrix}.
\end{equation}
For more information about the direct product group see \cite{Engo2003}. 


\subsection{Uncertainty}\label{ssec:lg_uncertainty}
In this paper we use Gaussian distributions to model the uncertainty in the sensor, state estimate, and system dynamics. As an example, let $\tilde{x} \in \CS{x}$ be a zero-mean, Gaussian random variable with covariance $P$ such that $\tilde{x}\sim \mathcal{N}\left(\mu=0, P \right)$; i.e.
\begin{equation}\label{eq:prob_x_tilde_def}
    p\left(\tilde{x}\right) = \eta\exp\left(-\frac{1}{2}\tilde{x}^\top P^{-1} \tilde{x}  \right)
\end{equation}
with $\eta$ denoting the normalizing coefficient. Gaussian distributions are defined on vector spaces. Since not every Lie group has a vector space structure (e.g. \SE{2} is not a vector space since scalar multiplication is not defined on the set) Gaussian distributions cannot be defined directly on every Lie group. However, they can be defined on the Cartesian algebraic space at the identity element of the Lie group and extended to the Lie group using the exponential map. Thus, the probability of an element of the Lie group is the probability of the corresponding element of the Cartesian algebraic space. For example, let $x = \ExpI{\G{x}}{\tilde{x}}$. The probability of $x$ is determined from the probability of $\tilde{x}$ as depicted in Fig.~\ref{fig:cgd}.

\begin{figure}[t]
\centering
    \includegraphics[width=0.9\linewidth]{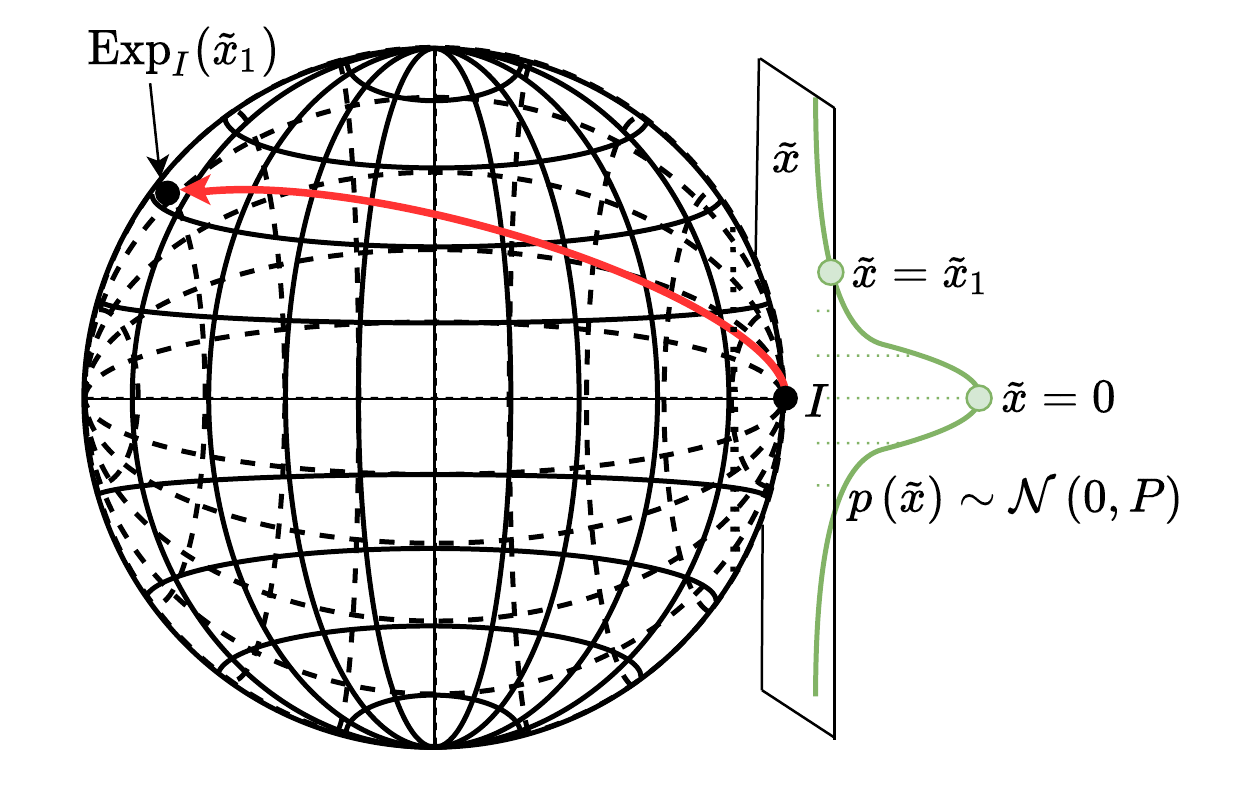}
    \caption{A depiction of the concentrated Gaussian distribution.}
     \label{fig:cgd}
\end{figure}

In order for the uncertainty to be indirectly defined over the entire Lie group, the Lie group is required to be connected (i.e. there exists a geodesic between any two elements). In other words, the exponential map at identity is surjective allowing the Gaussian distribution to extend to every element of the Lie group.

Depending on the connected Lie group, the exponential map may not be injective, which means that possibly an infinite number of elements of the Cartesian algebraic space will map to the same element of the Lie group. In this case, we require the uncertainty to have a concentrated Gaussian density (CGD) \cite{Wang2006}. The CGD is a zero mean Gaussian distribution that is tightly focused around the origin of the Cartesian algebraic space, where by tightly focused we mean that the majority of the probability mass is in a subset $U\subseteq \CS{x}$ centered around the origin where the exponential mapping from $U$ to $\G{x}$ is injective, and that the probability of an element not being in $U$ is negligible. This property allows us to ignore the probability of an element being outside of $U$.

The CGD can be centered at an element other than the identity element of the group using the group action provided that the Lie group is unimodular implying that the determinant of the group action is one. The unimodular property combined with the connected property allows the CGD to be mapped to the Lie group and moved to any element in the group via the group action without changing the probability mass density. For example, the target's state may be represented as
\begin{equation}\label{eq:x_def}
    x = \hat{x}~\ExpI{G_x}{ \tilde{x} },
\end{equation}
where $\hat{x} \in \G{x}$ is the target's state estimate and $\tilde{x}$ is the error state whose probability density function is defined in equation \eqref{eq:prob_x_tilde_def}. The exponential function at identity maps the random variable $\tilde{x}$ to the Lie group, and the state estimate moves the uncertainty to the target's state without changing the mass density of the uncertainty. These properties allow the probability of the state $x$ to be related to the corresponding probability of the error state $\tilde{x}$. Thus, the uncertainty distribution of $x$ is defined by the distribution of $\tilde{x}$ as
\begin{subequations}\label{eq:prob_x_def}
\begin{align} 
    p\left(x\right) & \stackrel{\triangle}{=} \eta\exp\left(-\frac{1}{2}\LogI{\G{x}}{\hat{x}^{-1} x}^\top P^{-1} \LogI{\G{x}}{\hat{x}^{-1} x} \right) \\
        & = \eta\exp\left(-\frac{1}{2}\tilde{x} ^\top P^{-1} \tilde{x}  \right) = p\left( \tilde{x}\right).
\end{align}
\end{subequations}
With a slight abuse of notation, we denote the probability density function (PDF) of the state $x$ as $x\sim \mathcal{N}\left(\hat{x}, P \right)$ where $\hat{x}$ is the state estimate and $P$ is the error covariance of the error state $\tilde{x}$. 

An advantage to representing the uncertainty in the Cartesian algebraic space is having a minimum representation of the uncertainty. For example, an element of the matrix group $\SE{2}$ has three dimensions but is represented by a $3\times 3$ matrix with nine elements. Representing the uncertainty directly on the set of $3\times 3$ matrices with nine elements would require the covariance to be $9\times 9$,  whereas the corresponding Cartesian algebraic space only has three components and the corresponding covariance matrix will be $3\times 3$.

For more information about representing uncertainty on Lie groups see~\cite{Kim2017,Long2013,Bourmaud2014,Bourmaud2013}.

\subsection{First Order Taylor Series and Partial Derivatives}\label{subsection:taylor_series}
Let $g,\hat{g}\in G$, and let $\tilde{g}\in\CS{G}$ be a small perturbation from the origin with the relation $g=\hat{g}\ExpI{G}{\tilde{g}}$. Also let $f:G\to G$ be an arbitrary function. The first order Taylor series of $f$ evaluated at $\hat{g}$ is 
\[
f\left(g\right)\approx f\left(\hat{g}\right)\ExpI{G}{\frac{\partial f}{\partial g}(\hat{g})~\tilde{g}},
\]
where $\frac{\partial f}{\partial g}(\hat{g})$
is the partial derivative of $f$ with respect to $g$ evaluated at $g=\hat{g}$. Using the definition and notation shown in \cite{Sola2018} and the relation defined in \eqref{eq:exp_relation}, the partial derivative
of $f$ with respect to $g$ is defined as 
\begin{align}\label{eq:lie_groups_review_derivative}
\frac{\partial f}{\partial g}  =\lim_{v\to0}\frac{ \LogI{G}{ f\left(g\right)^{-1} f\left(g\ExpI{G}{v}\right) }}{v}, 
\end{align}
where $v\in\CS{G}$. 
In Equations~\eqref{eq:lie_groups_review_derivative}, we abuse notation by denoting the vector of the numerator divided by each element of $v$, as the numerator divided by the vector $v$. Note that the limit in equation \eqref{eq:lie_groups_review_derivative} is taken in the Cartesian algebraic space instead of the Lie group, since the Cartesian algebraic space is a vector space where derivatives are well defined. 
\section{Overview}\label{sec:lgipdaf_overview}

The Lie group integrated probabilistic data association filter (LG-IPDAF) is designed to track a single dynamic target using a single sensor that observes the target. The target is modeled using a constant-velocity, white-noise-driven system model defined on Lie groups. The system model is defined in Section~\ref{sec:lgipdaf_system_model}.

We call the act of the sensor observing and producing measurements from the measurement space at a given instant of time a {\em sensor scan}. The sensor is assumed to detect the target with probability $P_D \in  \left[ 0,1 \right]$ where $P_D = 0$ means that the target is not in the sensor's field of view. It is assumed that a sensor produces at most one target originated measurement, called a true measurement.  We also assume that every sensor scan and all other measurements are non-target-originated measurements, called false measurements or clutter. The false measurements are assumed to be independent identically distributed (iid) with uniform spatial distribution, where the number of false measurements per sensor scan is modeled by the Poisson distribution
\begin{equation} \label{eq:lgipdaf_poisson_distribution}
    \mu_F\left(\phi\right) = \exp{\left(\lambda \mathcal{V}\right)}\frac{\left(\lambda \mathcal{V}\right)^\phi}{\phi !},
\end{equation}
where $\lambda$ is the spatial density of false measurements, $\mathcal{V}$ is the volume of the sensor's field of view, and $\phi$ is the number of false measurements. 

The LG-IPDAF represents the target mathematically using tracks. A track is a tuple $\mathcal{T}=\left(\hat{x}, P, \mathcal{CS}, \epsilon , L \right)$, where $\hat{x}$ is the state estimate of a target, $P$ is the corresponding error covariance, $\mathcal{CS}$ is the set of measurements associated to the track called the consensus set, $ \epsilon$ denotes the probability that the track represents a target and is called the track likelihood, and $L$ is the track label, a unique numerical label used to identify confirmed tracks. 

The LG-IPDAF does not assume that the target's initial state is known and relies on a track initialization and confirmation scheme to locate the target. Tracks are initialized from measurements. Since the sensor produces both target-originated measurements and false measurements, tracks generated by the LG-IPDAF can be created from true measurements and/or false measurements. Therefore, an initialized track can either represent a target or clutter. To identify the track that represents the target, the LG-IPDAF calculates the track likelihood. A track with a high track likelihood is confirmed to be a good representation of the target, and the confirmed track with the highest track likelihood is assumed to represent the target. 

As time progresses, the target moves and the sensor produces measurements. When new measurements are received, the LG-IPADF algorithm performs four steps in order: (1) the prediction step, (2) the data association step, (3) the update step, and (4) the track initialization step.

Let $t_k$ denote the current time and $t_{k^-}$ denote the time at the previous iteration. At the beginning of the prediction step, the track's state estimate, error covariance and track likelihood are at time $t_{k^-}$ and conditioned on the measurements up to time $t_{k^-}$. We denote these values respectively as $\hat{x}_{k^- \mid k^- }$, $P_{k^- \mid k^-}$, and $\epsilon_{k^- \mid k^-}$. During the prediction step the track's state estimate, error covariance and track likelihood are propagated forward in time using the system model. The propagated state estimate, error covariance and track likelihood at time $t_k$, conditioned on the measurement up to time $t_{k^-}$, are denoted $\hat{x}_{k \mid k^- }$, $P_{k \mid k^-}$, and $\epsilon_{k \mid k^-}$ respectively. The prediction step is discussed in detail in Section~\ref{sec:lgipdaf_prediction}.

During the data association step, the new measurements are associated to an existing track or given to a database that contains non-track-associated measurements. The data association algorithm is designed so that a target-originated measurement is associated to the track that represents the target with probability $P_G \in \left[ 0, 1 \right]$ provided that the track exists. This probability is used to construct a volume in the measurement space centered around the track's estimated measurement called the validation region. Any measurement that falls within a track's validation region is associated to the track. We denote the set of measurements associated to a track at time $t_k$ as $Z_k = \left\{ z_{k,j} \right\}_{j=1}^{m_k}$ where $z_{k,j}$ denotes the $j^{th}$ measurement at time $t_k$ associated with the track and $m_k$ denotes the number of measurements associated with the track at time $t_k$. The data association step is discussed in detail in Section~\ref{sec:lgipdaf_validation}.

During the update step, the new measurements associated with a track are used to update the track. For every associated measurement $\left\{ z_{k,j}\right\}_{j=1}^{m_k}$, the track's state estimate $\hat{x}_{k \mid k^-}$ is copied and updated using a distinct associated measurement to produced what we define to be the {\em split state estimates} $\left\{\hat{x}_{k\mid k,j}\right\}^{m_k}_{j=0}$ where $\hat{x}_{k\mid k,j}$ for $j>0$ is the state estimate after being updated with measurement $z_{k,j}$ and $\hat{x}_{k\mid k,0}=\hat{x}_{k \mid k^-}$ is the null hypothesis that none of the measurements originated from the target. Then the split state estimates $\left\{\hat{x}_{k\mid k,j}\right\}^{m_k}_{j=0}$ are fused together according to their probability of being the correct state estimate. In this way, the LG-IPDAF is never totally correct but never completely wrong as in the case of hard data association algorithms. After the update step, the propagated state estimate, error covariance and track likelihood at time $t_k$,  conditioned on the measurements up to time $t_k$, are denoted $\hat{x}_{k \mid k }$, $P_{k \mid k}$, and $\epsilon_{k \mid k}$ respectively. The update step is discussed in detail in Section~\ref{sec:lgipdaf_update}. 

During the track initialization step, non-track-associated measurements are used to initialize new tracks. This step is discussed in detail in Section~\ref{sec:lgipdaf_track_init}.  A depiction of a single iteration of the LG-IPDAF is shown in Fig.~\ref{fig:lgipdaf_pdaf}. 
\begin{figure}[t]
\centering
    \includegraphics[width=0.8\linewidth]{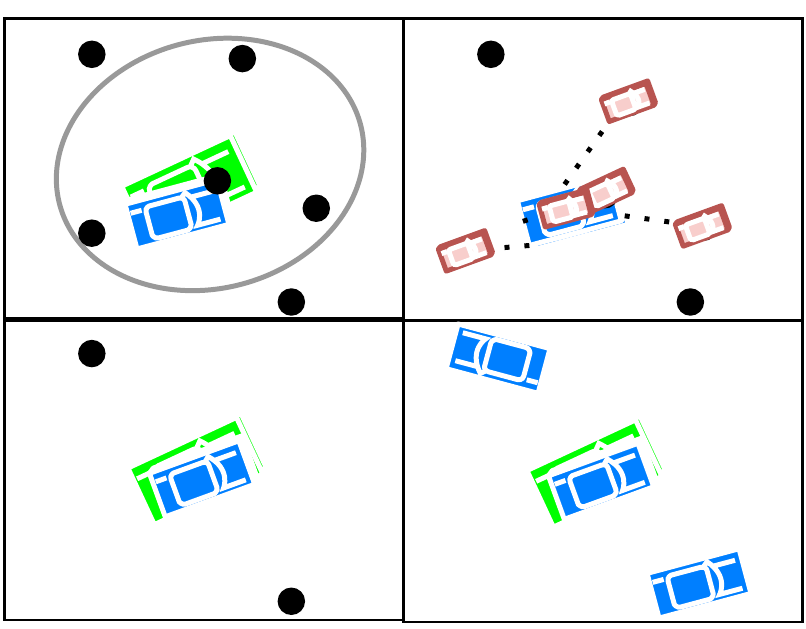}
    \caption{A depiction of a single iteration of the LG-IPDAF.  The large green car represents the target, the smaller blue cars represent the tracks, the black dots represent measurements, the gray ellipse represents the validation region and the small red cars represent the split state estimates. The top left image depicts the prediction and data association steps during which new measurements are received, the track is propagated in time, the validation region is constructed, and four measurements are associated to the track. The top right image shows the first part of the update step where the state estimate is split into five states: one for each associated measurement, and one for the null hypothesis that none of the measurements are target originated. The bottom left shows the rest of the update step where the split state estimates are fused together to form a single state estimate. The bottom right image shows the last step which initializes new tracks from non-track-associated measurements.}
     \label{fig:lgipdaf_pdaf}
\end{figure}

\subsection{Assumptions}
Our derivation of the LG-IPDAF uses the following assumptions:
\begin{enumerate}\label{apm:lgipdaf_assumptions}
    \item There exists a single target that can be observed by a single sensor and modeled by a constant-velocity, white-noise-driven system model defined in equation~\eqref{eq:sm_system_model}. \label{amp:lgipdaf_system_model}
    \item A sensor scan occurs whenever the sensor observes the measurement space. At every sensor scan there are $m_k$ validated measurements denoted $\left\{ z_{k,j}\right\}_{j=1}^{m_k}=Z_k$. \label{amp:lgipdaf_num_meas}
    \item At every scan there is at most one target originated measurement and all others are false (i.e. non target-originated measurements). \label{amp:lgipdaf_single_target_meas}
    \item The sensor detects the target with probability $P_D \in \left[0,1\right]$.\label{amp:lgipdaf_PD}
    \item The target originated measurement falls within the track's validation region with probability $P_G$ provided that the track represents the target. The probability $P_G$ is discussed in Section~\ref{sec:lgipdaf_validation}.\label{amp:lgipdaf_PG}
    \item The false measurements are independently identically distributed (iid) with uniform spatial density $\lambda$.\label{amp:lgipdaf_iid}
    \item The expected number of false measurements per sensor scan is modeled using the density function $\mu_F$. In this paper, $\mu_F$ denotes a Poisson distribution defined in equation~\eqref{eq:lgipdaf_poisson_distribution}.\label{amp:lgipdaf_mu}
    \item The past information about a track is summarized as $\left(\hat{x}_{k^- \mid k^-}, P_{k^- \mid k^-}, \epsilon_{k^- \mid k^-} \right)$ where $\hat{x}_{k^- \mid k^- }$, $P_{k^- \mid k^-}$, and $ \epsilon_{k^- \mid k^-}$ denote the track's state estimate, error covariance and track likelihood at the previous time and conditioned on the previous track-associated measurements. \label{amp:lgipdaf_past_info}
\end{enumerate}

The parameters $P_D$ and $\lambda$ can be statistically calculated from sensor data, and the parameter $P_G$ is selected by the user and discussed in Section~\ref{sec:lgipdaf_validation}.
\section{System Model} \label{sec:lgipdaf_system_model}

In this section we present the system model used in the LG-IPDAF and also derive the affinization of the system model that will be necessary when approximating Gaussian distributions. 

\subsection{System Model}

Let $x_k=\left(g_k,v_k\right)\in \G{x} \defeq G\times\CS{G}$ denote the target's state at time $t_k$, $t_{k^-}$ denote the time at the previous iteration, $t_\kt = t_k - t_{k^-}$ denote the time interval from the previous iteration to the current time, $\CS{x} \defeq \CS{G}\times\CS{G}\ni q_{\kt}=\left(q_\kt^g,q_\kt^v\right)\sim\mathcal{N}\left(0,Q\left( t_\kt\right)\right)$ denote the process noise modeled as a Wiener process \cite{Higham2001}, $z_k\in \G{s}$ denote the measurement at time $t_k$, $\CS{s} \ni r_k\sim\mathcal{N}\left(0,R\right)$ denote the measurement noise, where $\G{s}$ is the Lie group for the measurement space and $\CS{s}$ is the corresponding Cartesian algebraic space. The proposed discrete, time-invariant model is 
\begin{subequations}\label{eq:sm_system_model}
\begin{align}
x_k & =f\left(x_{k^-},q_\kt,t_\kt \right)\\
z_k & =h\left(x_k,r_k\right), \label{eq:sm_system_model_observation_function}
\end{align}
\end{subequations} 
where $f$ is the state transition function defined as 
\begin{subequations}\label{eq:sm_state_transition_function}
\begin{align}
&f\left(x_{k^-},q_\kt,t_\kt \right) \notag\\ & \triangleq (g_{k^-}, v_{k^-}) \ExpI{\G{x}}{ t_\kt v_{k^-}+q_{\kt}^{g}, q_\kt^{v}}\label{eq:state_transition_function_a}\\ 
& = \left(g_{k^-}\ExpI{G}{t_\kt v_{k^-} +q_{\kt}^{g} }, v_{k^-} + q_\kt^{v} \right)\label{eq:sm_state_transition_function_b}.
\end{align}
\end{subequations}
This form is similar to the system model defined in~\cite{Sjoberg2019}.

The definition of the observation function $h$ is dependent on the application and further generalization is not needed.  In~Section~\ref{sec:lgipdaf_examples}, we show how $h$ is defined when the target of interest is a car restricted to a plane, and where its pose (position and orientation) is observed. In that case,  the state is the Lie group $\SE{2}\times \CS{\SE{2}}$ and the measurement Lie group is $\SE{2}$, where the observation function is
\begin{equation}
    h\left(x_k,r_k\right) = g_k \ExpI{\G{\SE{2}}}{r_k},\label{eq:sm_observation_function}
\end{equation}
and where the noise satisfies $r_k\sim\mathcal{N}\left(0,R\right) \in \CS{\SE{2}}$.

\subsection{System Affinization}

A Gaussian random variable applied to an affine function remains Gaussian while a Gaussian random variable applied to a nonaffine function is no longer Gaussian. Depending on the Lie group, the system model can be affine or nonaffine. In the nonaffine case and under the assumption that the signal-to-noise ratio (SNR) is high, the Gaussian structure of the system model uncertainties can be well preserved with little loss of information by affinizing the system model when propagating and updating the uncertainty as is commonly done with the extended Kalman filter.

The system model is approximated as affine by computing its first order Taylor series at the points
\begin{align}
\zeta_{f_\kt} &\defeq \left(x_{k^-}=\hat{x}_{k^-}, q_{\Delta}=0,t_\kt \right) \\
\zeta_{h_k} &\defeq \left(x_{k}=\hat{x}_{k}, r_{k}=0 \right),
\end{align}
according to Subsection~\ref{subsection:taylor_series}.  What we mean by affine, is that the propagation of the uncertainty is affine. This computation requires the calculation of the state transition function Jacobians and the observation function Jacobians. 

The Jacobians for the state transition function can be defined generically, and the Jacobians for the observation function are application dependent.  However, we will define the observation function Jacobians for our running example. 
\begin{lemma} \label{lem:sm_jacobians}
Given the discrete time-invariant model in Equations~\eqref{eq:sm_system_model}, \eqref{eq:sm_state_transition_function} and~\eqref{eq:sm_observation_function} the Jacobians evaluated at $\zeta_{f_\kt}$ and $\zeta_{h_k}$ are
\begin{subequations}\label{eq:sm_system_jacobians}
\begin{align}
F_{\kt} & =\left.\frac{\partial f}{\partial x}\right|_{\zeta_{f_\kt}}  =\begin{bmatrix}\Ad{G}{\ExpI{G}{t_\kt \hat{v}_{k^-}}^{-1}} & \jr{G}{t_\kt \hat{v}_{k^-} }t_\kt\\
0_{n \times n} & I_{n \times n} 
\end{bmatrix} \label{eq:sm_F_x}\\
G_{\kt} & \defeq \left.\frac{\partial f}{\partial q}\right|_{\zeta_{f_{\kt}}}  =\begin{bmatrix}\jr{G}{t_{\kt} \hat{v}_{i}} & 0_{n \times n}\\
0_{n \times n} & I_{n \times n}
\end{bmatrix} \label{eq:sm_G} \\
H_{k} & = \left.\frac{\partial h}{\partial x}\right|_{\zeta_{h_k}}  = \begin{bmatrix}
I_{n \times n} & 0_{n \times n}
\end{bmatrix}\\
V_{k} &= \left.\frac{\partial h}{\partial r}\right|_{\zeta_{h_k}}  = I_{n \times n},
\end{align}
\end{subequations}
where $n$ is the dimension of the target's pose, $0_{n\times n}$ is the $n \times n$ zero matrix, and $I_{n \times n}$ is the $n \times n$ identity matrix. 

Consequently, if $\hat{x}_{k^-} \in \G{x}$ is a state estimate that is close to $x_{k^-}$, and if the error state between $x_{k^-}$ and $\hat{x}_{k^-}$ is defined as $\tilde{x}_{k^-}\defeq\LogI{\G{x}}{\hat{x}_{k^-}^{-1}  x_{k^-} } \in \CS{x}$, then the evolution of the system can be described by the ``affinized system''
\begin{subequations}\label{eq:sm_lienarized_system}
\begin{align}
x_k & \approx f\left(\hat{x}_{k^-},0,t_\kt \right) \ExpI{\G{x}}{F_{\kt}\tilde{x}_{k^-} + G_\kt q_\kt} \\
z_k & \approx h\left(\hat{x}_k,0\right)\ExpI{\G{s}}{ H_k \tilde{x}_k + V_k r_k }.
\end{align}
\end{subequations}
\end{lemma}

\begin{proof}
We will prove the expression for $F_\kt$, and leave the derivation of the other Jacobians, which are similar, to the reader. 
Let $\tau = \left( \tau^g, \tau^v \right) \in \CS{x}$ denote the perturbation of the state. 
Then, using the definition of the derivative in equation~\eqref{eq:lie_groups_review_derivative}
{\small
\begin{equation*}
\frac{\partial f}{\partial x_{k^-}} =\underset{\tau \to 0}{\lim} \frac{\LogI{\G{x}}{f\left(x_{k^-},0,t_{\kt} \right)^{-1} f\left( x_{k^-}\ExpI{\G{x}}{\tau},0,t_{\kt} \right)}}{\tau}.
\end{equation*}
}
Substituting the definition of the state transition function in equation~\eqref{eq:sm_state_transition_function} yields
{\small
\begin{align*}
\frac{\partial f}{\partial x_{k^-}} =&\underset{\tau\to0}{\lim}\frac{1}{\tau} 
\text{Log}^{\G{x}}_I\left[ \left( g_{k^-}\ExpI{G}{t_{\kt}v_{k^-}} , v_{k^-}\right)^{-1} \right. \\
& \left. \left( g_{k^-} \ExpI{G}{\tau^g} \ExpI{G}{t_{\kt}v_{k^-} + t_{\kt} \tau ^v} , v_{k^-} + \tau^v \right) \right] \\
= &\underset{\tau\to0}{\lim}\frac{1}{\tau}\text{Log}^{\G{x}}_I\left[ \left( \ExpI{G}{t_{\kt} v_{k^-}}^{-1} g^{-1}_{k^-} g_{k^-}\ExpI{G}{\tau^g} \right. \right. \\
& \left. \left. \ExpI{G}{t_{\kt} v_{k^-} + t_{\kt}\tau^v}, v_{k^-} + \tau^v - v_{k^-}  \right) \right] \\
= & \underset{\tau\to0}{\lim}\frac{1}{\tau} \text{Log}^{\G{x}}_I\left[ \left( \ExpI{G}{t_{\kt} v_{k^-}}^{-1} \ExpI{G}{\tau^g} \right. \right. \\
& \left. \left. \ExpI{G}{t_{\kt} v_{k^-} + t_{\kt}\tau^v}, \tau^v \right) \right] 
\end{align*}
}

Using the property of the adjoint in equation~\eqref{eq:lie_groups_review_adjoint_relation} and the property of the right Jacobian in equation \eqref{eq:Jr_property} gives
{\small
\begin{align*}
\frac{\partial f}{\partial x_{k^-}} =&  \underset{\tau\to0}{\lim} \frac{1}{\tau} \text{Log}^{\G{x}}_I\left[\left(\ExpI{G}{t_{\kt} v_{k^-}}^{-1} \ExpI{G}{\tau^g} \right. \right. \\
& \left. \left. \ExpI{G}{t_{\kt} v_{k^-}} \ExpI{G}{\jr{G}{t_{\kt}v_{k^-}} t_{\kt} \tau^v},\tau^v \right) \right] \\
=&  \underset{\tau\to0}{\lim} \frac{1}{\tau} \text{Log}^{\G{x}}_I\left[\left(\ExpI{G}{\Ad{G}{\ExpI{G}{t_{\kt}v_{k^-}}^{-1}} \tau^g} \right. \right. \\
& \left. \left. \ExpI{G}{\jr{G}{t_{\kt}v_{k^-}} t_{\kt} \tau^v},\tau^v \right) \right] \\
     =& \underset{\tau\to0}{\lim}\frac{1}{\tau}\left( \text{Log}^{G}_I \left( \ExpI{G}{\Ad{G}{\ExpI{G}{t_{\kt}v_{k^-}}^{-1}} \tau^g} \right. \right.\\
     & \left. \left. \ExpI{G}{\jr{G}{t_{\kt}v_{k^-}} t_{\kt} \tau^v} \right) , \tau^v\right) 
\end{align*}
}
The portion of the derivative corresponding to $\tau^g$ is computed as 
{\small
\begin{align}
\frac{\partial f}{\partial g_{k^-}} & =\underset{\tau^g \to0}{\lim}\frac{\left( \LogI{G}{ \ExpI{G}{\Ad{G}{\ExpI{G}{t_{\kt}v_{k^-}}^{-1}} \tau^g}  } ,0 \right)}{\tau^g}\notag \\
 &=\underset{\tau^g \to0}{\lim}\frac{\left(\Ad{G}{\ExpI{G}{t_{\kt}v_{k^-}}^{-1}} \tau^g, 0\right)}{\tau^g}
 =  \begin{bmatrix}
        \Ad{G}{\ExpI{G}{t_{\kt}v_{k^-}}^{-1}} \\
        0_{n \times n} 
    \end{bmatrix}, \label{eq:sm_F_g}
\end{align}
}

\par\noindent and the portion of the derivative corresponding to $\tau^v$ is computed as
\begin{align}
  \frac{\partial f}{\partial v_{k^-}} & =\underset{\tau^v \to0}{\lim}\frac{\left(\LogI{G}{\ExpI{G}{\jr{G}{t_{\kt}v_{k^-}} t_{\kt} \tau^v}},\tau^v \right)}{\tau^v} \notag \\
 &=\underset{\tau^v \to0}{\lim}\frac{\left(\jr{G}{t_{\kt}v_{k^-}} t_{\kt} \tau^v,\tau^v \right)}{\tau^v}
 =  \begin{bmatrix}
        \jr{G}{t_{\kt}v_{k^-}} t_{\kt} \\ 
        I_{n \times n}
    \end{bmatrix}. \label{eq:sm_F_v} 
\end{align}
Combining equations \eqref{eq:sm_F_g} and \eqref{eq:sm_F_v} yields \eqref{eq:sm_F_x}.
\end{proof}

\section{The Lie Group Integrated Probabilistic Data Associate Filter}
\label{sec:lgipdaf} \label{sec:lgipdaf_main}
This section gives detailed derivation of the key elements of the LG-IPDA filter.  In particular, the prediction step in Section~\ref{sec:lgipdaf_prediction}, the data association step in Section~\ref{sec:lgipdaf_validation}, the measurement update step in Section~\ref{sec:lgipdaf_update}, and new track initialization in Section~\ref{sec:lgipdaf_track_init}.

\subsection{Prediction Step}\label{sec:lgipdaf_prediction}

The prediction step of the LG-IPADF is similar to the prediction step of the indirect Kalman filter with the addition of propagating the track likelihood \cite{Sola2017,Sjoberg2019}. We begin by introducing additional notation.
Let $Z_{0:i}$ denote the set of track associated measurements from the initial time to time $t_i$. Let $\epsilon$ denote a Bernoulli random variable that represents the probability that the track represents the target. The probability that the track represents a target conditioned on $Z_{0:k^-}$ is the track likelihood and is denoted at the previous time step as $ \epsilon_{k^- \mid k^-} \defeq p\left(\epsilon_{k^-}\mid Z_{0:k^-}\right)$. Conversely, we denote the probability that the track does not represent a target conditioned on $Z_{0:k^-}$ as $ p\left(\epsilon_{k^-}=F\mid Z_{0:k^-}\right)=1-p\left(\epsilon_{k^-}\mid Z_{0:k^-}\right)$.

Define the probability of the track's previous state conditioned on the previous track-associated measurements and it representing a target as
\begin{align}\label{eq:lgipdaf_state_previous}
    p\left(x_{k^-} \mid \epsilon_{k^-}, Z_{0:k^-} \right) &\defeq \eta \exp \left( -\frac{1}{2} \tilde{x}_{k^- \mid k^-}^\top  P_{k^-\mid k^-}^{-1}  \tilde{x}_{k^- \mid k^-} \right),
\end{align}
where $\tilde{x}_{k- \mid k^-} = \LogI{\G{x}}{\hat{x}^{-1}_{k^- \mid k^-} x_{k^-} }$ is the error state, $\hat{x}_{k^- \mid k^-}$ is the state estimate and $P_{k^- \mid k^-}$ is the error covariance at time $t_{k^-}$ conditioned on the measurements $Z_{0:k^-}$. The terms $\hat{x}_{k^- \mid k^-}$, $P_{k^- \mid k^-}$, and $\epsilon_{k^- \mid k^-}$ are the track's state estimate, error covariance and track likelihood at the beginning of the propagation step.

To derive the prediction step, we need to construct the Gaussian approximation of the probability of the track's current state conditioned on the track's previous state, the track representing the target, and the previous measurements denoted $p\left(x_{k}\mid x_{k^{-}}, \epsilon_{k^-} , Z_{0:k^-} \right)$.

\begin{lemma}\label{lem:lgipdaf_prob_prop_state}
Given that Assumptions~\ref{amp:lgipdaf_system_model} and \ref{amp:lgipdaf_past_info} hold, then the Gaussian approximation of $p\left(x_{k}\mid x_{k^{-}}, \epsilon_{k^-} , Z_{0:k^-} \right)$ is 
{\small 
\begin{multline}\label{eq:lgipdaf_approx_state_propagation}
 p\left(x_{k}\mid x_{k^{-}}, \epsilon_{k^-} , Z_{0:k^-} \right) 
 \\
 \approx \eta\exp\left(-\frac{1}{2}\left(\tilde{x}_{k\mid k^{-}}-F_{\kt}\tilde{x}_{k^{-} \mid  k^{-}}\right)^{\top} \bar{Q}_{\kt}\left(\tilde{x}_{k\mid k^{-}}-F_{\kt}\tilde{x}_{k^{-} \mid  k^{-}}\right)\right),
\end{multline}
}
where
\begin{align}
\tilde{x}_{k\mid k^-} & = \LogI{\G{x}}{ \hat{x}_{k \mid  k^-}^{-1}  x_k  } \notag \\
\hat{x}_{k \mid  k^-} & = f\left(\hat{x}_{k^-\mid k^-},0,t_{\kt} \right) \notag \\
\bar{Q}_{\kt} & = G_{\kt}Q\left( t_{\kt} \right) G_{\kt}^{\top}, \label{eq:lgipdaf_q_bar}
\end{align}
and where $f$ is the state transition function defined in equation~\eqref{eq:sm_state_transition_function}, $Q\left(t_\kt\right)$ is the process noise covariance, and the Jacobians $F_{\kt}$ and $G_{\kt}$ are defined in equation~\eqref{eq:sm_system_jacobians} and evaluated at the point $\zeta_{f_{\kt}} = \left(\hat{x}_{k^-\mid k^-}, 0,t_{\kt} \right)$.
\end{lemma}
\begin{proof}
The probability $p\left(x_{k}\mid x_{k^{-}}, \epsilon_{k^-} , Z_{0:k^-} \right)$ is approximated as Gaussian by computing the first and second moments of the propagated error state, using the affinized system defined in \eqref{eq:sm_lienarized_system}. From Equation~\eqref{eq:sm_lienarized_system} the affinized state transition function is
\begin{equation}\label{eq:lgipdaf_linearized_state_transition_function_belief}
x_{k}\approx\underbrace{f\left(\hat{x}_{k^{-} \mid  k^{-}},0,t_{\kt} \right)}_{\hat{x}_{k\mid k^{-}}}\ExpI{\G{x}}{\underbrace{F_{\kt}\tilde{x}_{k^{-} \mid  k^{-}}+G_{\kt}q_{\kt}}_{\tilde{x}_{k\mid k^{-}}}},
\end{equation}
where $x_{k}$, $\hat{x}_{k \mid k^-} = f\left(\hat{x}_{k^{-} \mid  k^{-}},0,t_{\kt} \right)$, $\tilde{x}_{k \mid k^-} = F_{\kt}\tilde{x}_{k^{-} \mid  k^{-}}+G_{\kt}q_{k}$ are respectively the predicted state, the predicted state estimate, and the predicted error state all conditioned on the track's previous state, the track representing the target, and the previous measurements.

Since the probability $p\left(x_{k}\mid x_{k^{-}}, \epsilon_{k^-} , Z_{0:k^-} \right)$  is conditioned on a specific value of the previous state, the previous error state $\tilde{x}_{k^- \mid k^-}$ is not a random variable.  Thus, the first and second moments of the propagated error state are 
\begin{align*}
\text{E}\left[\tilde{x}_{k\mid k^{-}}\right] & =\text{E}\left[F_{\kt}\tilde{x}_{k^{-} \mid  k^{-}}+G_{\kt}q_{\kt}\right]\\
 & =F_{\kt}\tilde{x}_{k^{-} \mid k^{-}}\\
\text{cov}\left[\tilde{x}_{k\mid k^{-}}\right] & =G_{\kt}Q\left( t_{\kt} \right) G_{\kt}^{\top}=\bar{Q}_{\kt},
\end{align*}
and the approximate Gaussian PDF is given in equation~\eqref{eq:lgipdaf_approx_state_propagation}.
\end{proof}

\begin{lemma}
\label{lem:lgipdaf_belief_measurement}
Suppose that Assumptions~\ref{amp:lgipdaf_system_model} and \ref{amp:lgipdaf_past_info} hold,
then the propagation of the track's state estimate and error covariance are
\begin{subequations}\label{eq:lgipdaf_pred_state_estimate_covariance}
\begin{align}
\hat{x}_{k \mid k^{-}} & =f\left(\hat{x}_{k^{-} \mid k^{-}},0,t_{\kt}\right),\\
P_{k \mid k^{-}} & =F_{\kt}P_{k^{-} \mid k^{-}}F_{\kt}^{\top}+G_{\kt}Q\left(t_{\kt}\right) G_{\kt}^{\top}
\end{align}
\end{subequations}
and the probability of the track's current state conditioned on the track representing the target and the previous measurements is
\begin{align}\label{eq:lgipdaf_prob_x_curr_cond_prev_meas}
    p\left(x_{k} \mid \epsilon_{k^-}, Z_{0:k^-} \right) =
    \eta \exp\left( - \frac{1}{2} \tilde{x}_{k \mid k^-}^\top  P_{k \mid k^{-}}^{-1} \tilde{x}_{k \mid k^-} \right)
\end{align}
where
$\tilde{x}_{k \mid k^-} = \LogI{\G{x}}{\hat{x}_{k \mid k^{-}}^{-1} x_k }$.
\end{lemma}
\par\noindent{\em Proof:}  
See Appendix~\ref{app:lgipdaf_proof_prediction}.

The track likelihood is modeled using a Markov process as described in \cite{Musicki1994}, and its propagation is
\begin{align*}
\epsilon_{k \mid k^-} \defeq p\left( \epsilon_{k} \mid Z_{0:k^-} \right) & = \sigma\left(t_{\kt} \right)p\left( \epsilon_{k^-} \mid Z_{0:k^-}\right),
\end{align*}
where the probability $\sigma\left(t_{\kt} \right) \in \left[ 0,1 \right]$ is chosen to represent how the track likelihood can change in time due the target being occluded, leaving the sensor surveillance region, etc. 

Therefore, at the end of the prediction step, we know
the state estimate $\hat{x}_{k \mid k^-}$ and error covariance  and $P_{k \mid k^-}$, and the 
probabilities $p\left(x_{k} \mid \epsilon_{k^-}, Z_{0:k^-} \right)$ and $p\left( \epsilon_{k} \mid Z_{0:k^-} \right)$.
\subsection{Data Association}\label{sec:lgipdaf_validation}

In this section we derive the data association algorithm that associates new measurements to tracks. Let $\psi$ denote a Bernoulli random variable that indicates that the measurement originated from the target. The estimated measurement of a track with state estimate $\hat{x}_{k \mid k^-}$ is
\begin{equation}\label{eq:lgipdaf_estimated_meas}
    \hat{z}_{k} \triangleq h\left( \hat{x}_{k \mid k^-}, 0 \right),
\end{equation}
where $h$ is the generic observation function defined in equation \eqref{eq:sm_system_model}.

The validation region is a volume in measurement space centered around the track's estimated measurement. The volume is selected such that a target-originated measurement has probability $P_G$ to fall within the track's validation region provided that the track represents the target. A measurement that falls within the validation region of a track is called a validated measurement and is associated to the track.  Otherwise, the measurement is given to a database that stores non-track-associated measurements. Computation of the validation region is complicated by the fact that the measurements and the target's state are elements of Lie groups and do not have a vector space structure.

\begin{lemma}\label{lem:lgipdaf_val_ttp}
Suppose that Assumptions~\ref{amp:lgipdaf_system_model} and \ref{amp:lgipdaf_past_info} hold, then the probability of measurement $z_k$ conditioned on it being target-originated, is given by
\begin{equation}\label{eq:lgipdaf_probability_y}
    p\left(z_{k} \mid \psi, \epsilon_{k}, Z_{0:k^-} \right) \approx\eta\exp\left(-\frac{1}{2}\tilde{z}_{k}^{\top}S_{k}^{-1}\tilde{z}_{k}\right),
\end{equation}
where the innovation $\tilde{z}_k$ and the innovation covariance $S_k$ are given by
\begin{subequations}\label{eq:lgipdaf_innovation}
\begin{align}
\hat{z}_k &= h\left(\hat{x}_{k\mid k^-},0\right) \\
\tilde{z}_{k} &=\LogI{\G{s}}{\hat{z}_k^{-1} z_k } \label{eq:lgipdaf_innovation_term} \\
S_{k} & =V_{k}RV_{k}^{\top}+H_{k}P_{k|k^-}H_{k}^{\top},\label{eq:lgipdaf_innovation_covariance}
\end{align}
\end{subequations}
and where the Jacobians $H_k$ and $V_k$ are defined in \eqref{eq:sm_system_jacobians} and evaluated at the point $\zeta_{h_k}=\left( \hat{x}_{k\mid k^-},0 \right)$.
\end{lemma}

\par\noindent{\em Proof:~}
Similar to the proof of Lemma~\ref{lem:lgipdaf_belief_measurement} in Appendix~\ref{app:lgipdaf_proof_prediction}.

Using equations \eqref{eq:lgipdaf_probability_y} and \eqref{eq:lgipdaf_innovation}, we define the metric $d_{\mathcal{V}}:\G{s} \times \G{s} \to\mathbb{R}$ as 
\begin{equation}\label{eq:lgipdaf_validation_region_metric}
d_{\mathcal{V}}\left(z_{k},\hat{z}_{k}\right)  =\nu_{k}^{\top}S_{k}^{-1}\nu_{k},
\end{equation}
where the innovation covariance is used to normalize the metric.  Thus, the metric $d_{\mathcal{V}}$ is the sum of $m$ squared Gaussian random variables where $m$ is the
dimension of the measurement space, and the values of the metric are distributed according to
a chi-square distribution with $m$ degrees of freedom.

The validation region is defined as the set 
\[
\text{val}\left(\hat{z}_{k},\tau_G\right)\defeq\left\{ z\in \G{s}\,\mid \,d_V\left(z,\hat{z}_{k}\right)\leq\tau_{G}\right\},
\]
where the parameter $\tau_{G}$ is called the gate threshold. A measurement that falls within this validation region is associated to the track whose estimated measurement is $\hat{z}_k$. This implies that a validation region is constructed for each track.

The volume of the validation region is defined as
\begin{equation}
\mathcal{V}_k=c_{m}\left|\tau_{G}S_{k}\right|^{1/2},\label{eq:lgipdaf_volume_validation_region}
\end{equation}
where $c_{m}$ is the volume of the unit hypersphere of dimension
$m$ calculated as 
\[
c_{m}=\frac{\pi^{m/2}}{\Gamma\left(m/2+1\right)},
\]
with $\Gamma$ denoting the gamma function \cite{Bar-Shalom2011}. It is worth noting that the volume of the validation region is dependent on the error covariance through the innovation covariance $S_k$. Therefore, the validation region contains information on the quality of the state estimate. This concept will be used in Section~\ref{sec:lgipdaf_update}.

The gate probability is
\begin{equation}\label{eq:lgipdaf_gate_prob}
P_{G}=\int_{\mathcal{V}_k}  p\left(z \mid \psi, \epsilon_{k}, Z_{0:k^-} \right)\,dz,
\end{equation}
and is the value of the chi-square cumulative distribution function (CDF) with parameter $\tau_{G}$ \cite{Bar-Shalom2011}. 
Let $\theta_k$ denote a Bernoulli random variable that the measurement is target originated and inside the validation region. Using \eqref{eq:lgipdaf_probability_y} and \eqref{eq:lgipdaf_gate_prob}, the probability of a measurement conditioned on the measurement being inside the validation region, being target-originated, the track representing the target and the previous measurements is
\begin{equation} \label{eq:lgipdaf_prob_meas_val}
p\left(z_{k}\mid \theta_k, \epsilon_k, Z_{0:k^-} \right)=P_{G}^{-1}p\left(z_{k}\mid \psi, \epsilon_k, Z_{0:k^-} \right).    
\end{equation}

The data association step of the LG-IPDAF can be extended to multiple targets and tracks by assigning a measurement to every track whose validation region the measurement falls in without taking into account joint associations. This is done by copying the measurement for every track it is associated with and giving a copy to each track the measurement is associated with and treating track associated measurements of one track independently of track associated measurements of another track. 

At the end of the data association step, we have the track-associated measurements $Z_k = \left\{ z_{k,j} \right\}_{j=1}^{m_k}$ for each track. 

\subsection{Measurement Update Step}\label{sec:lgipdaf_update}

According to Assumption~\ref{amp:lgipdaf_single_target_meas}, at most one measurement originates from
the target every sensor scan.  Thus, given the set of new validated measurements $Z_{k}=\left\{ z_{k,j}\right\}_{j=1}^{m_k}$, either one of the measurements is target originated or none of them are. This leads to different possibilities. As an example, all of the measurements could be false, or the measurement $z_{k,1}$ could be the target-originated measurement and all others false, or $z_{k,2}$ could be the target-originated measurement and all others false, etc.
These different possibilities are referred to as association events denoted $\theta_{k,j}$ where the subscript $j>0$ means that the $j^{th}$ validated measurement is target originated and all others are false, and where $j=0$ means that all of the validated measurements are false. Hence, there are a total of $m_k+1$ association events. 

\begin{lemma} \label{lem:lgipdaf_association_event}
Suppose that Assumptions~\ref{amp:lgipdaf_system_model}-\ref{amp:lgipdaf_past_info} hold, and that $Z_{k}=\left\{ z_{k,j}\right\}_{j=1}^{m_k}$ is the set of validated measurements at time $t_k$, and let $P_D$ be the probability of detection, $P_G$ the gate probability, $\lambda$ the spatial density of false measurements,
$p\left(z_{k, j} \mid \psi, \epsilon_k, Z_{0:k^-} \right)$ be defined by Equation \eqref{eq:lgipdaf_probability_y}, and let
\begin{equation} \label{eq:lgipdaf_prob_L}
\mathcal{L}_{k,j} \defeq \frac{P_D}{\lambda}p\left(z_{k, j} \mid \psi, \epsilon_k, Z_{0:k^-} \right),
\end{equation}
then the probability $\beta_{k,j} \defeq p\left(\theta_{k,j}\mid \epsilon_{k},Z_{0:k}\right)$ of an association event $\theta_{k,j}$ conditioned on the measurements $Z_{0:k}$, and the track representing a target, is given by 
\begin{equation} \label{eq:lgipdaf_association_probabilities}
\beta_{k,j}=\begin{cases}
\frac{\left(1-P_{G}P_{D}\right)}{1-P_{D}P_{G} + \sum_{j=i}^{m_k}\mathcal{L}_{k,i}} & j=0, \\
\frac{\mathcal{L}_{k,j}}{1-P_{D}P_{G} + \sum_{i=1}^{m_k}\mathcal{L}_{k,i}} & j=1,\ldots,m_{k}.
\end{cases}
\end{equation}
\end{lemma}
\par\noindent{\em Proof:}
See Appendix~\ref{app:lgipdaf_association_event}.

Using the association events, we can update the track's state estimate $\hat{x}_{k \mid k^-}$ and error covariance $P_{k \mid k^-}$ conditioned on each association event in order to generate the split tracks. This update step is similar to the Kalman filter's update step. The standard Kalman filter algorithm updates the state estimate directly using vector space arithmetic.  However, since not every Lie group has a vector space structure, we have to approach the update step differently.

Recall that  $p\left(x_{k} \mid \epsilon_k, Z_{0:k^-} \right) = p\left( \tilde{x}_{k} \mid \epsilon_k, Z_{0:k^-} \right)$ with the relation $x_k = \hat{x}_{k \mid k^-}\ExpI{\G{x}}{\tilde{x}_{k \mid k^-}}$. The error state $\tilde{x}_{k \mid k^-} \sim \gauss{\mu _{k \mid k^-}}{P_{k \mid k^-}}$ has a vector space structure. This allows us to update the error state using a method similar to the Kalman filter. 
We denote the updated error state and its corresponding mean and error covariance conditioned on $\theta_{k,j}$, the track representing the target, and all track-associated measurements as $\tilde{x}^-_{k \mid k,j}\sim \mathcal{N}\left(\mu^-_{k \mid k ,j}, P^{c^-}_{k\mid k,j} \right)$. Using the updated error state, the probability of the track's state conditioned on the association event $\theta_{k,j}$ and the track representing the target and all the track associated measurements, is $p\left(x_{k} \mid \theta_{k,j}, \epsilon_k,  Z_{0:k} \right) = p\left( \tilde{x}^-_{k} \mid \theta_{k,j}, \epsilon_k, Z_{0:k} \right)$ where $x_k = \hat{x}_{k \mid k^-} \ExpI{\G{x}}{\tilde{x}^-_{k \mid k} }$.

The updated error state may have a non-zero mean. To reset the error state's mean to zero we add $\mu^-_{k \mid k ,j}$  to the state estimate $\hat{x}_{k \mid k^-}$ via the exponential map, and the error covariance is modified accordingly. We denote the updated and reset error state and its corresponding mean and error covariance conditioned on $\theta_{k,j}$ as $\tilde{x}_{k \mid k,j}\sim \mathcal{N}\left(\mu_{k \mid k ,j}=0, P^{c}_{k\mid k,j} \right)$. Using the reset error state, we have $p\left(x_{k} \mid \theta_{k,j}, \epsilon_k,  Z_{0:k} \right) = p\left( \tilde{x}_{k} \mid \theta_{k,j}, \epsilon_k, Z_{0:k} \right)$ with the relation $x_k = \hat{x}_{k \mid k} \ExpI{\G{x}}{\tilde{x}_{k \mid k} }$. In summary, we update the mean of the error state, and then reset it to zero by adding it onto the state estimate. This process is the update step of the indirect Kalman filter and is presented in the following lemma.

\begin{lemma}\label{lem:lgipdaf_single_update}
Suppose that Assumptions~\ref{amp:lgipdaf_system_model}, \ref{amp:lgipdaf_single_target_meas},~and \ref{amp:lgipdaf_past_info} hold, and define $Z_k = \left\{ z_{k,j}\right\}_{j=1}^{m_k}$ as the set of validated measurements at time $k$, then 
the Gaussian approximation of the probability of the split track $x_{k,j}$ 
is
\begin{equation}\label{eq:lgipdaf_single_update_prob_split_track}
    p\left(x_{k,j} \mid \theta_{k,j}, \epsilon_k, Z_{0:k} \right)=\eta\exp\left(\tilde{x}_{k|k,j}^{\top}\left(P^c_{k|k,j}\right)^{-1}\tilde{x}_{k|k,j}\right),
\end{equation}
where
\begin{align}
    \tilde{x}_{k|k,j} &= \LogI{\G{x}}{ \hat{x}^{-1}_{k \mid k,j}  x_{k,j}} \\
    \hat{x}_{k|k,j} &= \begin{cases} \hat{x}_{k \mid k^{-}}\ExpI{\G{x}}{ \mu^{-}_{k|k,j}}, & j\neq 0 \\
    \hat{x}_{k|k^-}, & j=0
    \end{cases} \label{eq:lgipdaf_iku_reset_mean} \\
    P^c_{k|k,j} &= 
        \begin{cases} 
        \jr{\G{x}}{\mu^{-}_{k|k,j}} P^{c^-}_{k|k}  \jr{\G{x}}{\mu^{-}_{k|k,j}}^\top, & j\neq 0 
        \\
        P_{k|k^-}, & j=0
        \end{cases} \label{eq:lgipdaf_iku_cov_resnt_mean} \\ 
    \mu^{-}_{k|k,j} & =K_{k}\nu_{k,j}\label{eq:lgipdaf_single_update_mu}\\
    \nu_{k,j} & = \LogI{\G{s}}{h\left(\hat{x}_{k \mid k^{-}},0\right) ^{-1} z_{k,j}}     
        \label{eq:lgipdaf_single_update_innovation_term}\\
    K_{k} & =P_{k \mid k^{-}}H_{k}^{\top}S_{k}^{-1} \label{eq:lgipdaf_single_update_kalman_gain}\\
    S_{k} & =H_{k}P_{k \mid k^{-}}H_{k}^{\top}+V_{k}RV_{k}^{\top}\label{eq:lgipdaf_single_update_innovation_cov}\\
    P^{c^-}_{k|k} & = \left(I-K_{k} H_{k}\right)P_{k \mid k^{-}} \label{eq:lgipdaf_iku_normal_cov_update},
\end{align}
and where $R$ is the measurement noise covariance, and the Jacobians $H_{k}$ and $V_{k}$ are defined in \eqref{eq:sm_system_jacobians} and evaluated at the point $\zeta_{h_k} =\left(\hat{x}_{k \mid k^-},0 \right)$. 
\end{lemma}
\par\noindent{\em Proof: }
See Appendix~\ref{proof:lgipdaf_single_update}.

Equation~\eqref{eq:lgipdaf_single_update_mu} is the updated mean of the error state before being reset to zero, Equation~\eqref{eq:lgipdaf_single_update_innovation_term} is the innovation term, 
Equation~\eqref{eq:lgipdaf_single_update_kalman_gain} is the Kalman gain, Equation~\eqref{eq:lgipdaf_single_update_innovation_cov} is the innovation covariance, and Equation~\eqref{eq:lgipdaf_iku_normal_cov_update} is the updated error covariance before the updated mean is reset. Equation~\eqref{eq:lgipdaf_iku_reset_mean} adds the mean of the error state to the state estimate essentially resetting the error state's mean to zero, and Equation~\eqref{eq:lgipdaf_iku_cov_resnt_mean} is the covariance update as a consequence of resetting the error state's mean to zero.

Using the probabilities of the split tracks defined in equation~\eqref{eq:lgipdaf_single_update_prob_split_track} and the probabilities of each association event defined in equation~\eqref{eq:lgipdaf_association_probabilities}, the probability of the track's state conditioned on it representing the target and all track-associated measurements is calculated using the theorem of total probability:
\begin{equation}\label{eq:lgipdaf_prob_state_cond_exist_meas}
p\left(x_{k} \mid \epsilon_{k}, Z_{0:k} \right)  =\sum_{j=0}^{m_{k}}p\left(x_{k,j}\mid\theta_{k,j},\epsilon_{k}, Z_{0:k} \right)\beta_{k,j}.
\end{equation}
In essence, the probability of the track's state is the weighted average of the split tracks' state probabilities where the weight is the probability of the corresponding association event. The process of splitting tracks could be repeated every time new measurements are received which would lead to an exponential growth of split tracks. This
process would quickly become computationally and memory expensive. To keep the problem manageable, the split tracks are fused together using the smoothing property of conditional expectations discussed in \cite{Bar-Shalom2001}. 

Normally if the track's state estimate is expressed in Euclidean space, the smoothing property of conditional expectations would indicate that the state estimates should be fused together according to the equation
\begin{equation} \label{eq:lgipdaf_euclidean_fuse}
    \hat{x}_{k \mid k} = \sum_{j=0}^{m_k} \hat{x}_{k \mid k, j}\beta_{k,j},
\end{equation}
as in the original probabilistic data association filter.
However, this approach does not work with arbitrary Lie groups since not every Lie group has a vector space structure. Instead we have to solve equation \eqref{eq:lgipdaf_prob_state_cond_exist_meas} indirectly by posing the problem in the Cartesian algebraic space. This is accomplished by using the relation in equation~\eqref{eq:prob_x_def} to note that
\begin{equation}\label{eq:lgipdaf_relation_error_state}
    p\left(x_{k,j} \mid \epsilon_{k}, Z_{0:k} \right)=p\left(\tilde{x}^{-}_{k,j} \mid \epsilon_{k}, Z_{0:k} \right),
\end{equation} where $\tilde{x}^-_{k,j}$ is the error state conditioned on $\theta_{k,j}$ after update but before its mean is reset to zero as discussed in Lemma~\ref{lem:lgipdaf_single_update}. We use this version of the error state since we can fuse the means $\mu^-_{k\mid k,j}$, defined in equation~\eqref{eq:lgipdaf_single_update_mu}, together in a similar way as in equation~\eqref{eq:lgipdaf_euclidean_fuse}.

Using the relation in equation~\eqref{eq:lgipdaf_relation_error_state}, an equivalent expression to equation \eqref{eq:lgipdaf_prob_state_cond_exist_meas} is 
\begin{equation}
 p\left(\tilde{x}^{-}_{k\mid k} \mid \epsilon_k, Z_{0:k} \right)  =\sum_{j=0}^{m_{k}}p\left(\tilde{x}^{-}_{k \mid k,j}\mid\theta_{k,j},\epsilon_{k}, Z_{0:k} \right)\beta_{k,j},
\end{equation}
with the relations $x_k = \hat{x}_{k \mid k^-} \ExpI{\G{x}}{\tilde{x}^{-}_{k \mid k}}$ and $x_{k,j} = \hat{x}_{k \mid k^-,j} \ExpI{\G{x}}{\tilde{x}^{-}_{k \mid k,j}}$.

In essence, the validated measurements are used to update the mean of the error state for each split track $\mu^-_{k\mid k,j}$ according to Lemma~\ref{lem:lgipdaf_single_update}. Since the means are elements of the Cartesian algebraic space (in the same tangent space), they can be added together using a weighted average to form a single mean $\mu^-_{k \mid k}$. The mean of the error state is then reset to zero by adding $\mu^-_{k \mid k}$ onto the state estimate $\hat{x}_{k\mid k^{-}}$ using the exponential map. This process is depicted in Fig.~\ref{fig:lgipdaf_geodesic_fuse}.

\begin{figure}[t]
\centering
    \includegraphics[width=0.5\linewidth]{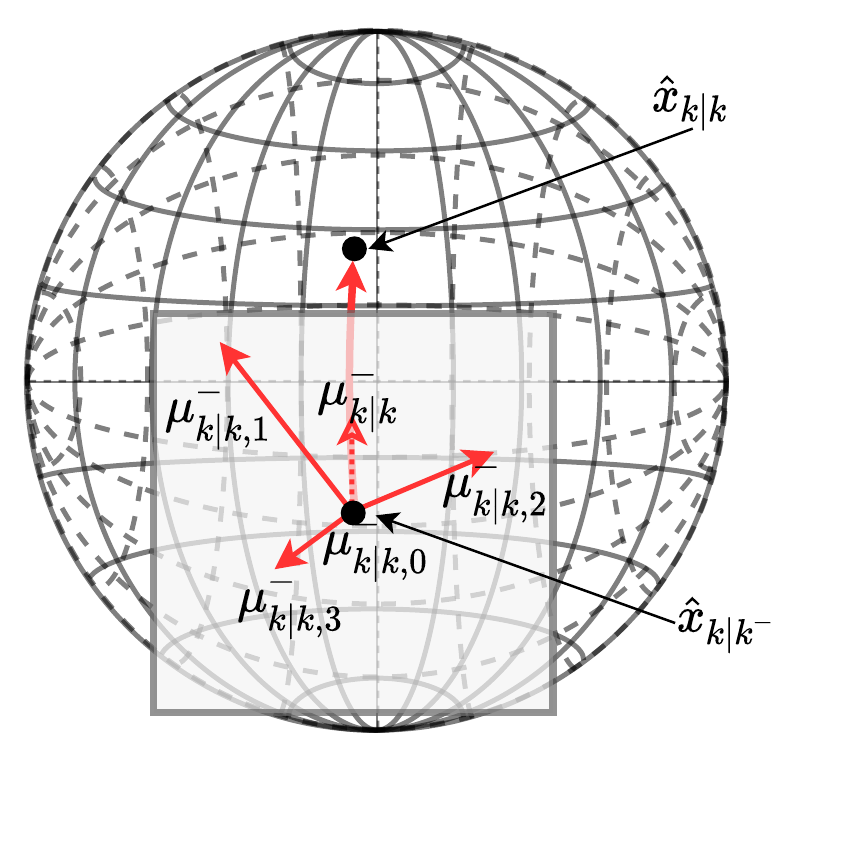}
    \caption{The error state means for each split track before reset $\mu^{-}_{k\mid k,j}$ are fused together using the association event probabilities $\beta_{k,j}$ in the tangent space to form the error state mean $\mu^{-}_{k\mid k}$. The mean is then used to update the track's state estimate from $\hat{x}_{k\mid k^{-}}$ to $\hat{x}_{k\mid k}$. }
     \label{fig:lgipdaf_geodesic_fuse}
\end{figure}

Using the smoothing property of conditional expectations, the expected value of the error state $\tilde{x}_{k \mid k}^-$ is
\begin{subequations}\label{eq:lgipdaf_expected_value}
\begin{align}
   \text{E}\left[ \tilde{x}^{-}_{k\mid k }\right] &=\text{E}\left[ \text{E}\left[ \tilde{x}^{-}_{k\mid k,j} | \theta_{k,j}, \epsilon_{k}\right]\right] \\
   & = \sum_{j=0}^{m_k} \text{E}\left[ \tilde{x}^{-}_{k\mid k,j} | \theta_{k,j}, \epsilon_{k}\right] \beta_{k,j}\\
   & = \sum_{j=0}^{m_k} \mu^{-}_{k\mid k,j} \beta_{k,j} \\
   & = \mu^{-}_{k\mid k}
  \end{align}
\end{subequations}
where $\beta_{k,j}$ is defined in Lemma~\ref{lem:lgipdaf_association_event} and $\mu^{-}_{k\mid k,j}$ is defined in Lemma~\ref{lem:lgipdaf_single_update}.

Using the updated mean of the error state, the covariance is 
\begin{equation*}
\text{cov}\left[\tilde{x}^{-}_{k\mid k}\right]=  \text{E}\left[\left(\tilde{x}^{-}_{k\mid k} - \mu^{-}_{k\mid k}\right)\left(\tilde{x}^{-}_{k\mid k} - \mu^{-}_{k\mid k}\right)^{\top}\right].
\end{equation*}
From this point the derivation follows from \cite{Bar-Shalom2011} and results in
\begin{equation}\label{eq:P_minus_k_given_k}
P^{-}_{k\mid k}=\beta_{k,0}P_{k\mid k^{-}}+\left(1-\beta_{k,0}\right)P_{k\mid k}^{c^-}+\tilde{P}_{k\mid k},
\end{equation}
where
\begin{subequations} \label{eq:ekf_subequations}
\begin{align}
   P^{c^-}_{k|k} & = \left(I-K_{k} H_{k}\right)P_{k|k^-} \\
   K_{k} & =P_{k|k^-}H_{k}^{\top}S_{k}^{-1}\\
   S_{k} & =H_{k}P_{k|k^-}H_{k}^{\top}+V_{k}R V_{k}^{\top}\\
   \tilde{P}_{k\mid k} &=K_{k}\left(\sum_{j=1}^{m_{k}}\beta_{k,j}\nu_{k,j}\nu_{k,j}^{\top}-\nu_{k}\nu_{k}^{\top}\right)K_{k}^{\top} \\   
   \nu_{k,j} & =\LogI{\G{s}}{h\left(\hat{x}_{k\mid k^-},0\right)^{-1} z_{k,j} }\\
   \nu_{k} & = \sum_{j=1}^{m_k}\beta_{k,j}\nu_{k,j} \\
   \mu^{-}_{k|k} & = K_{k}\nu_{k},
\end{align}
\end{subequations}
and where $P_{k\mid k^-}$ is the error covariance before the update step, $P^{c^-}_{k|k}$ is the error covariance of the error state $\tilde{x}^{-}_{k|k,j}$ derived in Lemma \ref{lem:lgipdaf_single_update}, $\tilde{P}_{k\mid k}$ is the covariance that captures the "spread of the means", and the Jacobians $H_{k}$ and $V_{k}$ are defined in \eqref{eq:sm_system_jacobians} and evaluated at the point $\zeta_{h_k} =\left(\hat{x}_{k \mid k^-},0 \right)$.

To reset the error state's mean to zero, the mean $\mu^{-}_{k\mid k}$ is added onto the state estimate $\hat{x}_{k\mid k^{-}}$ by forming a geodesic from $\hat{x}_{k\mid k^{-}}$ to $\hat{x}_{k\mid k}$ in the direction of $\mu^{-}_{k\mid k}$ as depicted in Fig~\ref{fig:lgipdaf_geodesic_fuse}. To derive this process, let $\tilde{x}^{-}_{k|k} = \mu^{-}_{k|k} + a_{k|k}$ where $a_{k|k} \sim \mathcal{N}\left(0, P^{c-}_{k|k} \right)$ contains the uncertainty in the error state.  Then under the assumption that $a_{k|k}$ is small and using the property of the right Jacobian defined in equation~\eqref{eq:Jr_property} we add $\mu^{-}_{k|k}$ onto $\hat{x}_{k|k^-}$ as follows:
\begin{subequations}
\begin{align}
x_{k|k} & = \hat{x}_{k\mid k^-}\ExpI{\G{x}}{ \mu^{-}_{k|k} + a_{k|k} } \\
 & \approx \underbrace{\hat{x}_{k|k^-}\ExpI{\G{x}}{ \mu^{-}_{k|k} }}_{\hat{x}_{k\mid k}}\ExpI{\G{x}}{ \underbrace{\jr{\G{x}}{\mu^{-}_{k|k}} a_{k|k}}_{\tilde{x}_{k\mid k}} },
\end{align}
\end{subequations}
where $\hat{x}_{k|k} = \hat{x}_{k|k^-}\ExpI{\G{x}}{ \mu^{-}_{k|k} }$ is the updated state estimate, and $\tilde{x}_{k \mid k} = \jr{\G{x}}{\mu^{-}_{k|k}} a_{k|k}$ is the updated and reset error state. 
The error covariance of the error state $\tilde{x}_{k|k}$ is 
\begin{align*}
    \text{cov}\left(\tilde{x}_{k|k}\right) & = \text{cov} \left( \jr{\G{x}}{\mu^{-}_{k|k}} a_{k|k} \right) \\
    & = \jr{\G{x}}{\mu^{-}_{k|k}} \text{cov} \left( a_{k|k} \right) \jr{\G{x}}{\mu^{-}_{k|k}}^\top  \\
    & = \jr{\G{x}}{\mu^{-}_{k|k}}   P^{-}_{k|k} \jr{\G{x}}{\mu^{-}_{k|k}}^\top \\
    & = P_{k|k},
\end{align*}
and therefore, $ \tilde{x}_{k|k}\sim \mathcal{N}\left(\mu_{k|k}=0, P_{k|k} \right) $.


\begin{lemma}\label{lem:lgipdaf_track_likelihood}
Suppose that Assumptions~\ref{amp:lgipdaf_system_model}-\ref{amp:lgipdaf_past_info} hold, and that $Z_k$ is the set of validated measurements, then the update for the track likelihood is  
\begin{equation} \label{eq:lgipdaf_existence_final}
   \epsilon_{k \mid k} \defeq p\left(\epsilon_{k} \mid Z_{0:k}\right) = \frac{1-\alpha_k}{1-\alpha_k  p\left(\epsilon_{k} \mid Z_{0:k^-}\right)}  p\left(\epsilon_{k} \mid Z_{0:k^-}\right),
\end{equation}
where 
\begin{equation} \label{eq:lgipdaf_alpha_weights} 
    \alpha_{k} = 
    \begin{cases}
    P_D P_G, & \text{if } m_k=0 \\
    P_D P_G - \sum_{j=1}^{m_{k}}\mathcal{L}_{k,j}, & \text{otherwise}
    \end{cases},
\end{equation}
and $\mathcal{L}_{k,j}$ is defined in equation \eqref{eq:lgipdaf_prob_L}.
\end{lemma}
\par\noindent{\em Proof:}
See Appendix~\ref{proof:lgipdaf_track_likelihood}.

Using the track likelihood, a track is either rejected after the update step if the track likelihood is below the threshold $\tau_{RT}$, confirmed to represent a target if the track likelihood is above the threshold $\tau_{CT}$, or neither rejected nor confirmed until more information is gathered with new measurements. In the case of tracking a single target, the confirmed track with the best track likelihood is assumed to represent the target. In the case of tracking multiple targets, every confirmed track is assumed to represent a different target. 

Therefore, we have the following main theorem that summarizes the results.
\begin{theorem}\label{thm:lgipdaf_update}
If Assumptions \ref{amp:lgipdaf_system_model}-\ref{amp:lgipdaf_past_info} hold and $Z_k = \left\{ z_{k,j} \right\}_{j=1}^{m_k}$ denotes the set of $m_k$ validated measurements at time $t_k$ and $\hat{x}_{k \mid k^-}$, $P_{k\mid k^-}$ and $p\left( \epsilon_{k} \mid Z_{0:k^-}\right)$ denote the track's state estimate, error covariance and track likelihood at time $t_k$ and conditioned on the previous measurements, then the track's updated state estimate, error covariance and track likelihood are

\begin{align}
    \hat{x}_{k|k} &= \hat{x}_{k|k^-}\ExpI{\G{x}}{  \mu^{-}_{k|k} } \\
    P_{k|k} &= \jr{\G{x}}{\mu^{-}_{k|k}} P^{-}_{k|k} \jr{\G{x}}{\mu^{-}_{k|k}}^\top \\
    p\left(\epsilon_{k} \mid Z_{0:k}\right) &= \frac{1-\alpha_k}{1-\alpha_k  p\left(\epsilon_{k} \mid Z_{0:k^-}\right)}  p\left(\epsilon_{k} \mid Z_{0:k^-}\right)
\end{align}
where $\mu^{-}_{k|k}$ is given in Equation~\eqref{eq:ekf_subequations}, $P^{-}_{k|k}$ is given in Equation~\eqref{eq:P_minus_k_given_k}, and $\alpha_k$ is defined in equation~\eqref{eq:lgipdaf_alpha_weights}.
\end{theorem}
\par\noindent{\em Proof:}
Follows directly from Lemmas~\ref{lem:sm_jacobians},  \ref{lem:lgipdaf_val_ttp}, \ref{lem:lgipdaf_association_event}, \ref{lem:lgipdaf_single_update}, and \ref{lem:lgipdaf_track_likelihood}. 
\subsection{Track Initialization}\label{sec:lgipdaf_track_init}

A track can be initialized in a variety of ways. We refer the reader to \cite{Bar-Shalom2011} for some of the common methods, and we adapt one of the methods to Lie groups in this section. We assume that the system model is as defined by equations~\eqref{eq:sm_system_model}, \eqref{eq:sm_observation_function} and \eqref{eq:sm_state_transition_function}.

Let $z_k$ and $z_i$ denote two non-track-associated measurements from distinct times where $z_k$ is a measurement from the current time and $z_i$ is a measurement from a previous time spatially close to the measurement $z_k$. We select two measurements that are close together since a moving target forms a continuous trajectory and produces measurements spatially close together. 

According to the observation function in Equation~\eqref{eq:sm_observation_function}, the target's pose is observed.  Therefore, we can estimate that $\hat{g}_k \approx z_k$. According to the state transition function in equation~\eqref{eq:sm_state_transition_function}, the target is assumed to have constant velocity. We solve for the velocity using the equation
\begin{equation}\label{eq:lgipdaf_vel_approx}
    v_k \approx  \frac{\LogI{\G{s}}{z_i ^{-1} z_k }}{t_k -t_i},
\end{equation}
where $t_k - t_i$ is the time interval between the measurements. 
The tracks current state estimate is set to $\hat{x}_{k \mid k} \approx \left( \hat{g}_{k \mid k}, \hat{v}_{k \mid k} \right)$ where $ \hat{g}_{k \mid k} = z_{k}$ and $ \hat{v}_k$ is solved for by equation~\eqref{eq:lgipdaf_vel_approx}.

The initial error covariance is specified by the application. As an example, for the simple example discussed in Section~\ref{sec:lgipdaf_examples}, we set the initial error covariance to $P = 5I_{2n \times 2n}$ to reflect the uncertainty in the initial state estimate and the high velocity of the moving car. The initial track likelihood can be set to $\epsilon_{k \mid k} = 0.5$ to reflect the fact that the new track is just as likely to represent clutter as it is to represent the target. 
\section{Example}\label{sec:lgipdaf_examples}

We demonstrate the LG-IPDAF in simulation by tracking a car that evolves on \SE{2} and whose system model is described in Section \ref{sec:lgipdaf_system_model}. The car is given four trajectories: a circular trajectory with a 10 meter radius, a Zamboni-like trajectory consisting of curves and straight lines, a spiral trajectory, and a straight-line trajectory. In the first trajectory the car is given a translational velocity of $10~m/s$ and an angular velocity of $1~rad/s$.  In the second trajectory the car is give a translational velocity of $10~m/s$ and an angular velocity that varies between zero to $1.5~rads/s$ to create the U-turns.  In the third trajectory the target is given an angular velocity of $1~rad/s$ and a translational velocity that increases from two to $10~m/s$.  In the fourth trajectory the car is given a translational velocity of $7~m/s$. We selected these trajectories to compare how well a constant-velocity model on \SE{2} (\SE{2}-CV) implemented using the LG-IPDAF versus a constant-velocity, linear model on $\mathbb{R}^2$ (LTI-CV) implemented using the IPDAF tracks a target that undergoes straight-line and non-straight-line motion. The sensor is modeled as being able to detect the car's pose with measurement noise covariance $R=\text{diag}\left(10^{-1},10^{-1},10^{-2}\right)$ and a surveillance region of 140 by 140 meters.

We set the following LG-IPDAF and IPDAF parameters: the spatial density of false measurements to $\lambda = 0.01$ implying that there is an average of 196 false measurements every sensor scan, the car's probability of detection to $P_D = 0.9$, gate probability to $P_G = 0.9$, track likelihood confirmation threshold to $\tau_{CT} = 0.7$, and track likelihood rejection threshold to $\tau_{RT} = 0.1$. We initialize new tracks from two unassociated neighboring measurements from different times with an initial error covariance of $P=5I_{2n \times 2n}$ and track likelihood to $\epsilon=0.2$. The process noise of the initialized tracks is $Q=\text{diag}\left( 1,1,0.1,1,1,0.1\right)dt$, where the simulated time step $dt=0.1$. The initialized process noise is large enough to account for the car's acceleration, but small enough to prevent many false measurements from being associated with the track. We also use a low initial track likelihood value due to the high number of false measurements per sensor scan making it more probable that an initialized track does not represent the target.

For each model (the \SE{2}-CV and LTI-CV models) and trajectory, we conduct a Monte-Carlo simulation consisting of 100 iterations, each 30 seconds long, and compute three statistical measures: track probability of detection (TPD), average Euclidean error (AEE) in position, and average confirmation time (ACT). Track probability of detection is the probability that the target is being tracked \cite{Gorji2011}. The average Euclidean error, is the confirmed track's average error in position. The average confirmation time is the average amount of time until the first track is confirmed. 

A depiction of the simulation for the circular, Zamboni, spiral, and straight-line trajectories are shown in Figs.~\ref{fig:lgipdaf_circular_zoom}, \ref{fig:lgipdaf_meander}, \ref{fig:lgipdaf_spiral} and \ref{fig:lgipdaf_straight}. The target's trajectory is shown as a black line, and the target's final pose is represented as the large black arrowhead. The confirmed tracks' trajectories are shown as piece-wise, green lines and the confirmed track's final pose is represented as a large green arrowhead. The target's measurements from the initial to the final time are represented as magenta dots. The red arrowheads represent different unconfirmed and non-rejected track's at the final time step of the simulation. The blue asterisks represent the false measurements received at the last time step of the simulation. During non-straight-line motions, the LTI-CV model struggles to track the target. This is shown in two ways: first by there not being a green trajectory (confirmed track's trajectory) during parts of the target's non-straight-line motion, and second the green trajectories drifting from the black trajectory (target's trajectory). The second case is prevalent in the circular trajectory for the LTI model indicating that many LTI tracks were initialized and confirmed, but quickly drifted off the circular trajectory and being pruned once their track likelihood fell below the pruning threshold.. 

The statistical measures from the experiment are in Table~\ref{tab:lgipdaf_results}. As stated in the table, the \SE{2}-CV model tracked the target significantly better for the circular,  Zamboni and spiral trajectories, and about the same as the LTI-CV model for the straight-line trajectory according to the track probability of detection measure. This is because the LTI-CV model struggles to track non-straight-line motion. In Fig.~\ref{fig:meander_zoom_lti} you can see that during the straight segments, the LTI-CV model tracks the target well, but looses the target as it turns. Table~\ref{tab:lgipdaf_results} also shows that the \SE{2}-CV had less error than the LTI-CV model.  However, on average the LTI-CV model had faster track confirmation time. We believe that this is due to the dimension of the measurement and state space. The smaller the dimension, the smaller the initial volume of the validation region (i.e. the initial overall uncertainty in the state estimate was less). The initial smaller uncertainty allowed the track likelihood for the LTI-CV model to increase faster than for the \SE{2}-CV model.

Since the main objective of this paper is deriving and presenting the LG-IPDAF algorithm, we did not do a detailed simulated analysis on the effects of changing the parameters used by the algorithm and in the simulation. The purpose of the experimental result is to show that there are cases where an LTI-CV model (or the original IPDAF algorithm) is insufficient. In essence, the LTI-CV model fails when it is no longer able to properly associate the correct measurement to the track. This occurs frequently when the target's position moves either outside of the LTI-CV model's validation region or if a series of false measurements are deemed more probable than the correct measurement according to the validation region. In the case of the circular trajectory, the target's position is often outside of the LTI-CV model's validation region causing the track to become lost. This is because the LTI-CV model predicts a straight-line motion that is tangent to the target's circular motion, causing the model to slowly drift away from the target's true position. 

The disadvantage of the LG-IPDAF is that it is more complex to implement and it is more computationally expensive compared to the IPDAF.
To help demonstrate the difference in computational complexity we did an analytical and experimental analysis of the two algorithms used in our experiment. On average, a single iteration of the IPDAF algorithm using the LTI-CV model took 4.4~seconds, whereas the LG-IPDAF algorithm using the \SE{2}-CV model required 13.2~seconds per iteration.  For the analytical analysis, let $x_d$ denote the dimension of the state, $g_d=x_d/2$ the dimension of the pose, $z_d$ the dimension of the measurement, $a$ the number of measurements received in a single sensor scan, and $m$ denote the number of associated measurements, then the approximate computational complexity of the IPDAF is $\mathcal{O}\left(z_d^2\left(a+3m\right) + 5x_d^3 \right)$ and the approximate computational complexity of the LG-IPDAF is $\mathcal{O}\left(z_d^3\log{\left(z_d\right)}\left(a+m\right)+ 2g_{d}^{3}\log\left(g_{d}\right)\right)$. 

The main difference in complexity between the two algorithms is the matrix logarithm and exponential functions which have complexity $\mathcal{O}\left(j^3\log\left(j\right)\right)$ where $j^2$ is the number of elements in the matrix. In the experiment, the data association step took the longest due to the number of measurements and dominated the computational complexity. If we compare just the computational complexity of the data association step by taking the ratio of the two we get \begin{equation*}
    \frac{3^3\log{\left(3\right)}a}{2^2a}\approx 3.3,
\end{equation*} 
where 3 is the dimension of the measurement space used by the \SE{2}-CV model since it includes orientation, and 2 is the dimension of the measurement space used by the LTI-CV model. This ratio is approximately the ratio of the simulated time $13.2/4.4=3$ seconds.

This computational complexity analysis is done without any optimization. Some of the more interesting Lie groups such as $\SE{n}$ and $\SO{n}$ have much more efficient methods to calculate the matrix exponential and logarithm, for example, the Rodriguez formula for $\SO{3}$.  

Even though the LG-IPDAF is more computationally complex, as shown by the experiments there are occasions when the IPDAF is inadequate and cannot properly track objects that have non-straight-line trajectories. This usually occurs in the presence of dense clutter or when the rate of receiving new measurements is too low to compensate for the target's non-straight-line motion resulting in the target's position drifting outside the track's validation region. It is in these instances that we recommend the use of the LG-IPDAF.



\begin{figure}[tbh]
     \centering
      \begin{subfigure}[t]{0.48\linewidth}
         \centering
         \includegraphics[width=\linewidth]{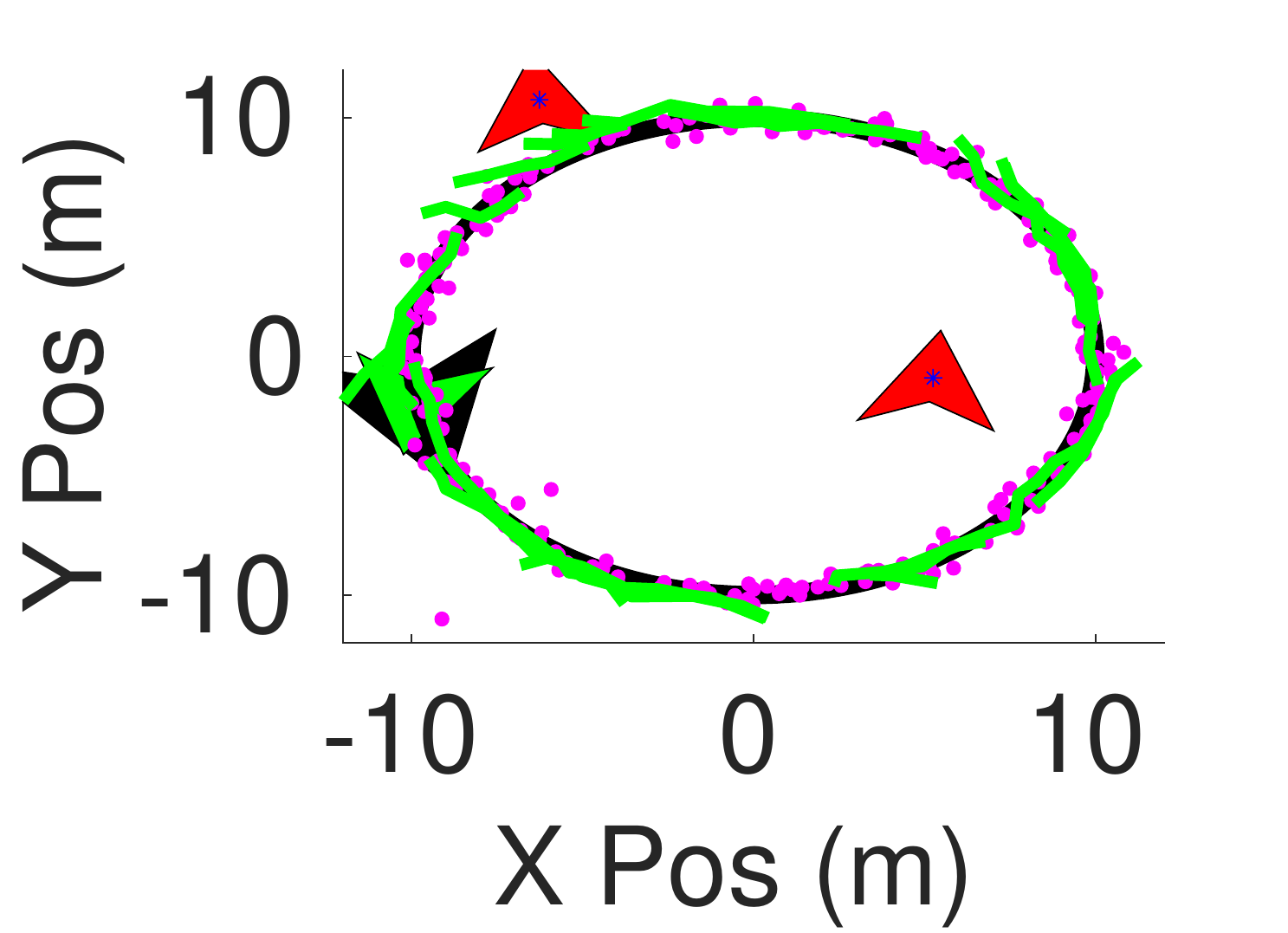}
         \caption{LTI-CV model}
         \label{fig:circular_zoom_lti}
     \end{subfigure}
     \hfill
     \begin{subfigure}[t]{0.48\linewidth}
         \centering
         \includegraphics[width=\linewidth]{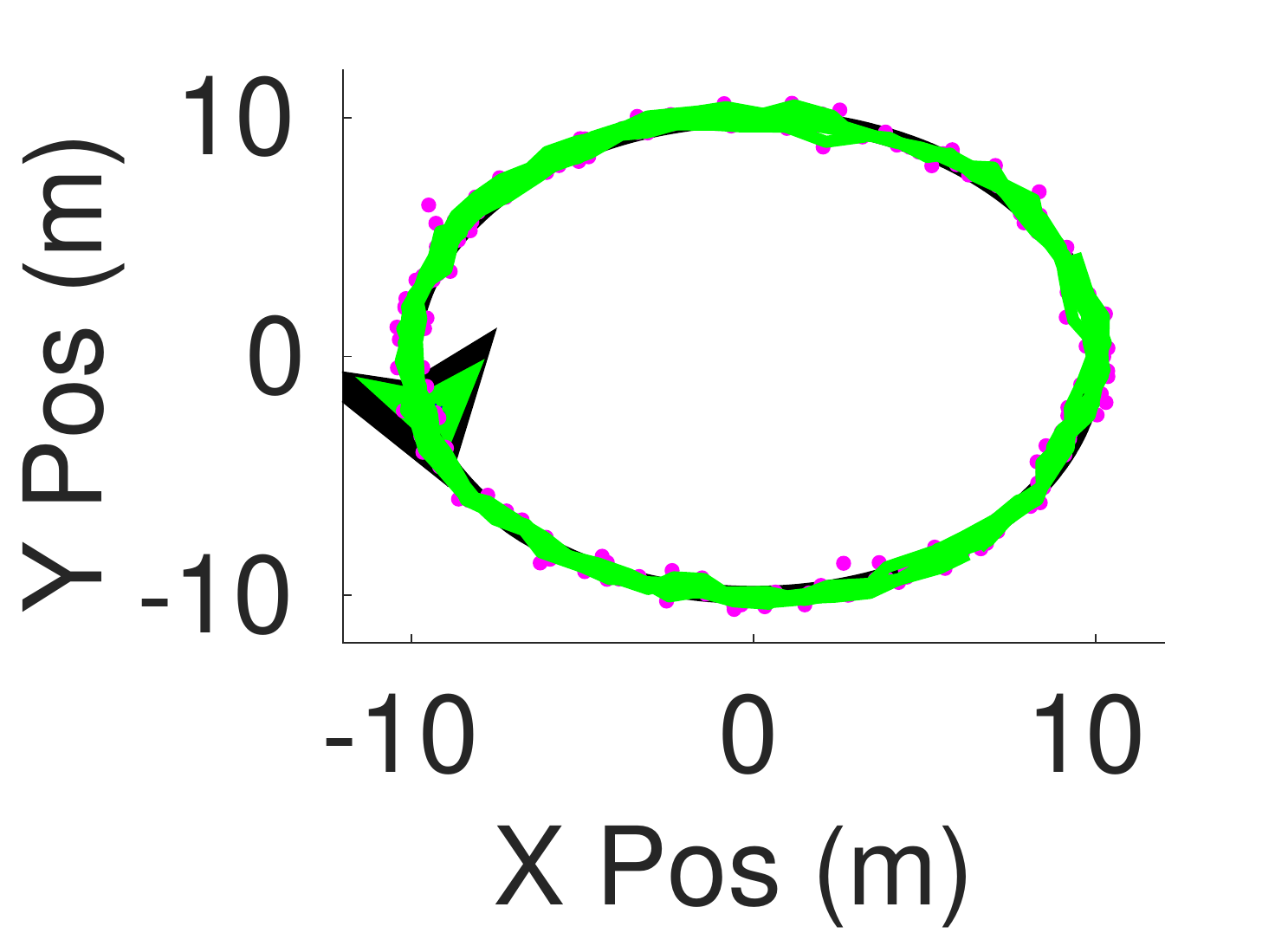}
         \caption{SE(2)-CV model}
         \label{fig:circular_soom_se3}
     \end{subfigure}
    
        \caption{Plots of the zoomed in circular trajectories for the LTI-CV and SE(2)-CV models.
        }
        \label{fig:lgipdaf_circular_zoom}
\end{figure}

\begin{figure}[tbh]
     \centering
     \begin{subfigure}[h]{0.48\linewidth}
         \centering
         \includegraphics[width=\linewidth]{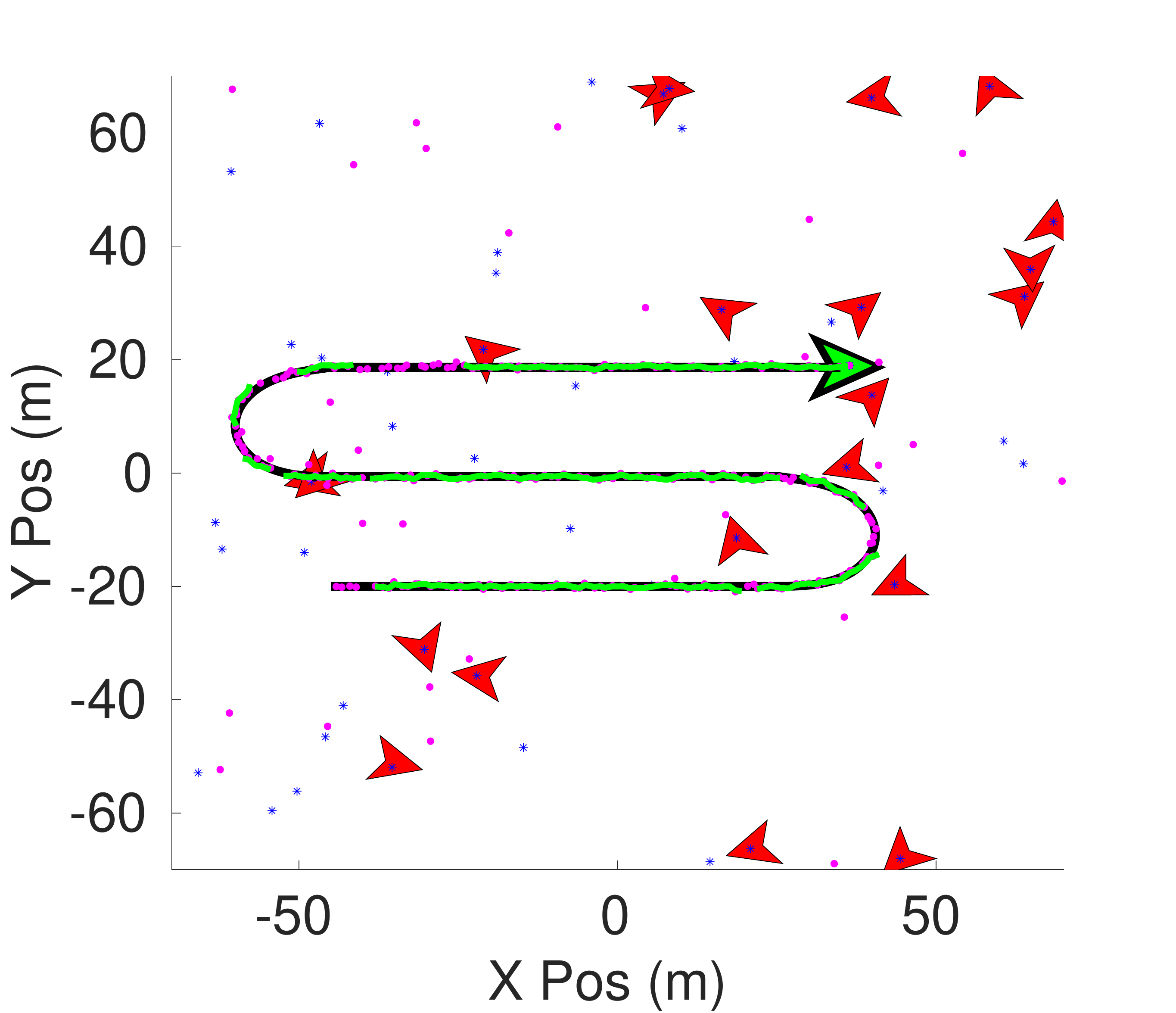}
         \caption{LTI-CV model}
         \label{fig:meander_zoom_lti}
     \end{subfigure}
     \hfill
     \begin{subfigure}[h]{0.48\linewidth}
         \centering
         \includegraphics[width=\linewidth]{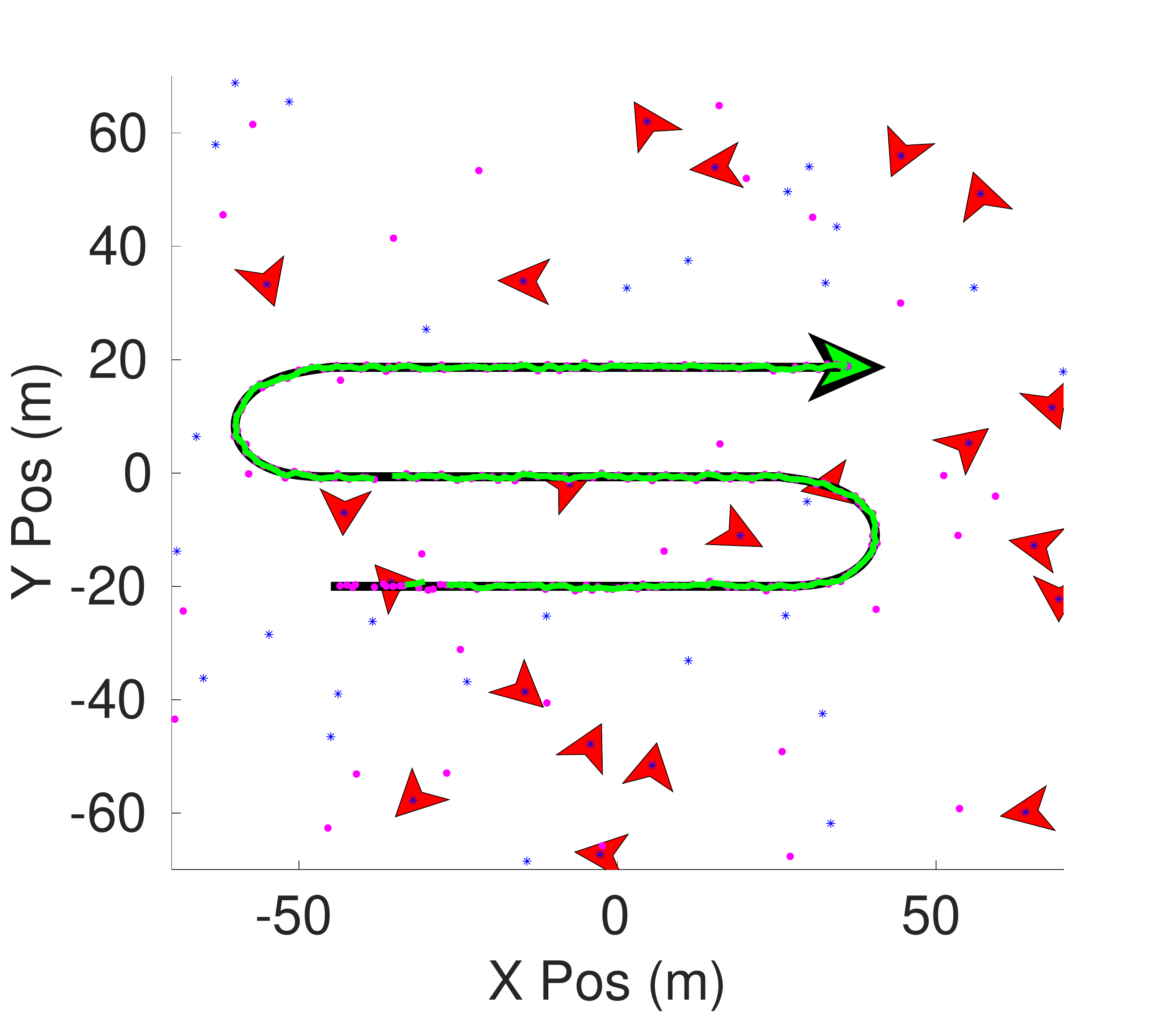}
         \caption{SE(2)-CV model}
         \label{fig:meander_soom_se3}
     \end{subfigure}

        \caption{Plots of the Zamboni-like trajectories for the LTI-CV and SE(2)-CV models.
        }
        \label{fig:lgipdaf_meander}
\end{figure}

\begin{figure}[tbh]
     \centering
     \begin{subfigure}[h]{0.48\linewidth}
         \centering
         \includegraphics[width=\linewidth]{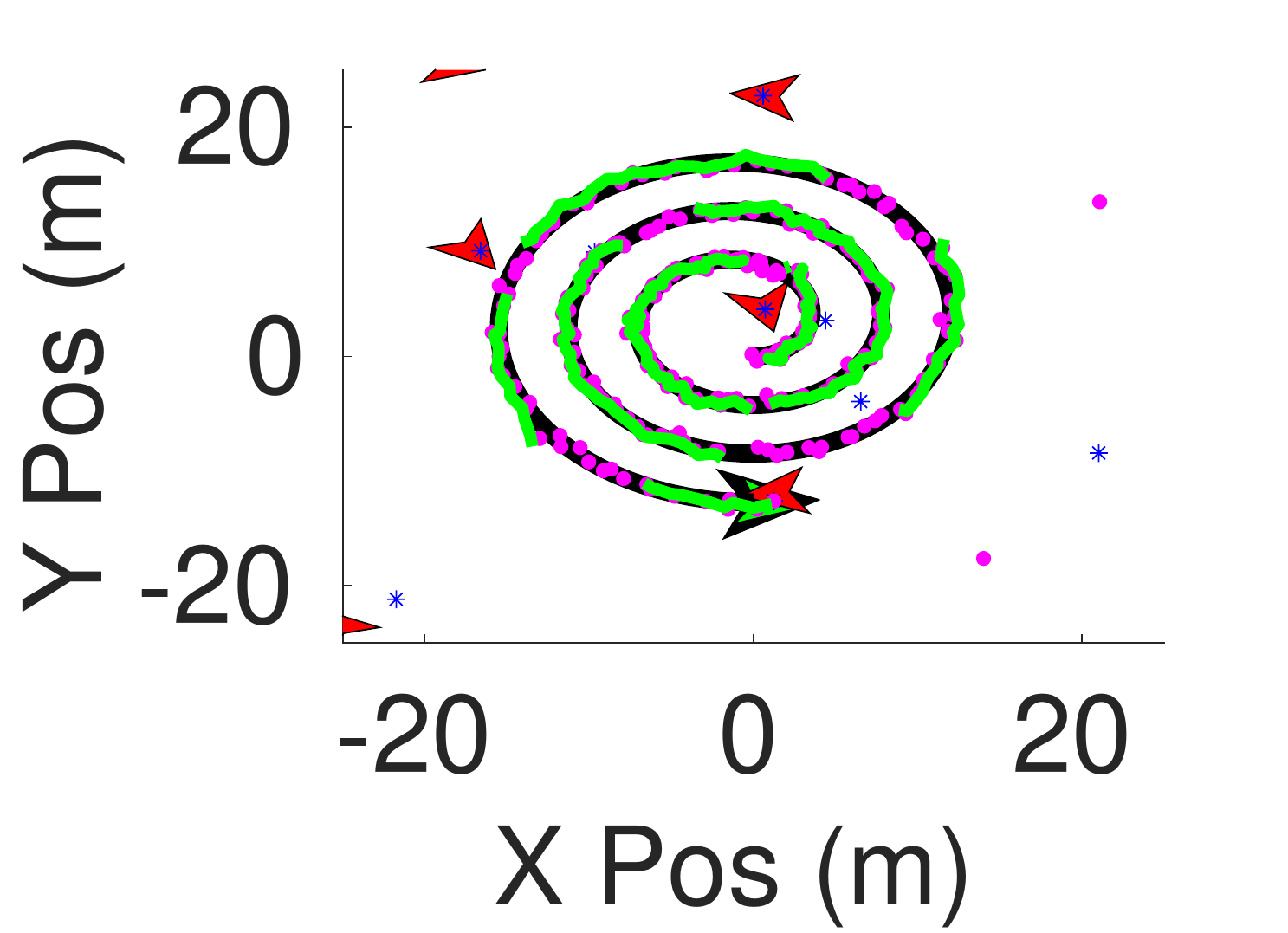}
         \caption{LTI-CV model}
         \label{fig:spiral_zoom_lti}
     \end{subfigure}
     \hfill
     \begin{subfigure}[h]{0.48\linewidth}
         \centering
         \includegraphics[width=\linewidth]{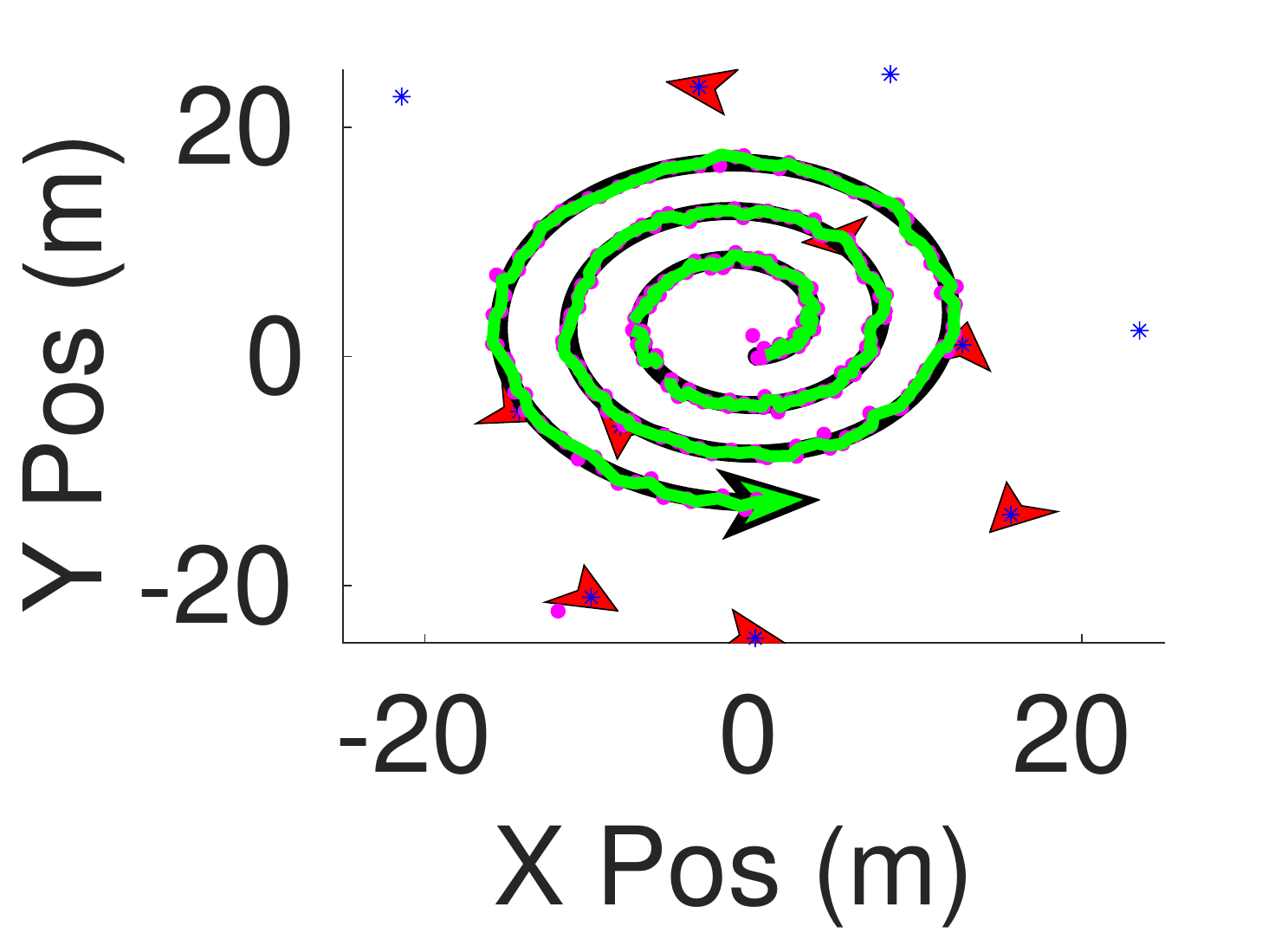}
         \caption{SE(2)-CV model}
         \label{fig:spiral_soom_se3}
     \end{subfigure}

        \caption{Plots of the Spiral-like trajectories for the LTI-CV and SE(2)-CV models.
        }
        \label{fig:lgipdaf_spiral}
\end{figure}

\begin{figure}[tbh]
     \centering
      \begin{subfigure}[b]{0.48\linewidth}
         \centering
         \includegraphics[width=\linewidth]{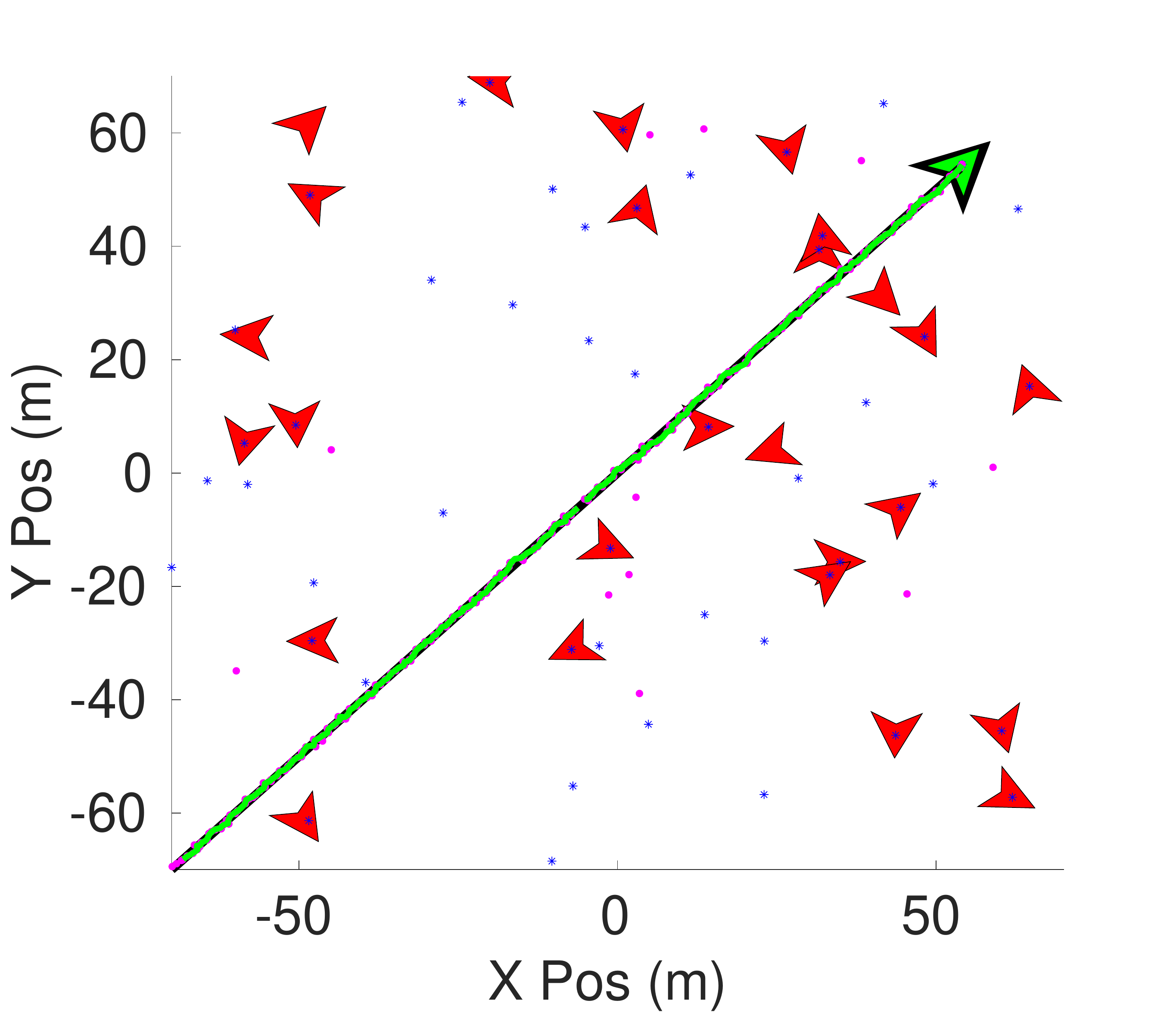}
         \caption{LTI-CV model}
         \label{fig:straight_zoom_lti}
     \end{subfigure}
     \hfill
     \begin{subfigure}[b]{0.48\linewidth}
         \centering
         \includegraphics[width=\linewidth]{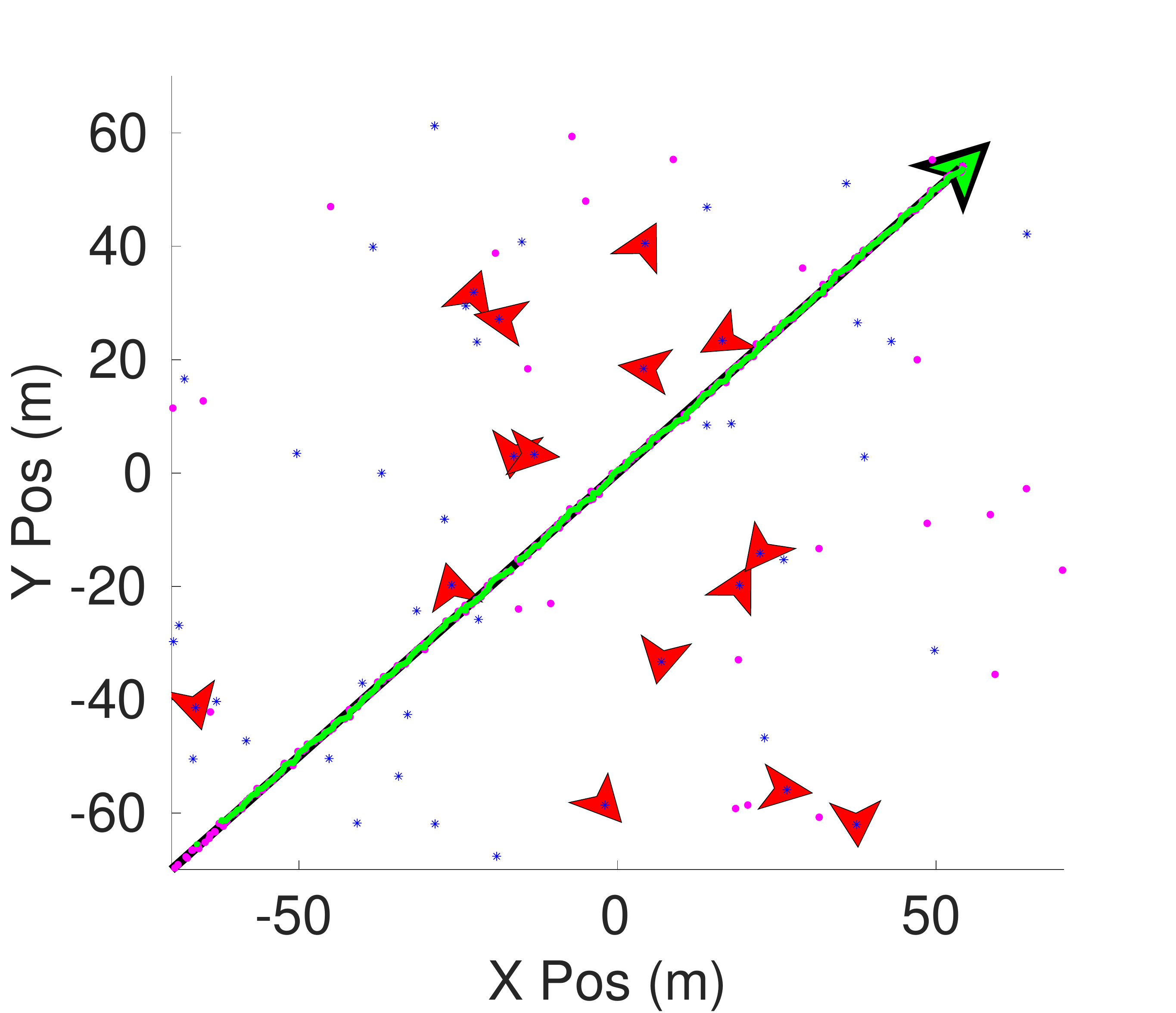}
         \caption{SE(2)-CV model}
         \label{fig:straight_soom_se3}
     \end{subfigure}
        \caption{Plots of straight-line trajectories for the LTI-CV and SE(2)-CV models.
        }
        \label{fig:lgipdaf_straight}
\end{figure}

\begin{table}[tbh]
\caption{Statistical measures from the experiment}
\label{tab:lgipdaf_results}
\resizebox{\linewidth}{!}{%
\begin{tabular}{|l|lll|lll|lll|lll|}
\hline
 & \multicolumn{3}{c|}{Circular} & \multicolumn{3}{c|}{Zamboni} & \multicolumn{3}{c|}{Spiral} & \multicolumn{3}{c|}{Straight} \\ \hline
 & \multicolumn{1}{l|}{TPD} & \multicolumn{1}{l|}{AEE} & ACT & \multicolumn{1}{l|}{TPD} & \multicolumn{1}{l|}{AEE} & ACT &
 \multicolumn{1}{l|}{TPD} & \multicolumn{1}{l|}{AEE} & ACT &\multicolumn{1}{l|}{TPD} & \multicolumn{1}{l|}{AEE} & ACT \\ \hline
SE(2) & \multicolumn{1}{l|}{\textbf{0.96}} & \multicolumn{1}{l|}{\textbf{0.29}} & 0.90 & \multicolumn{1}{l|}{\textbf{0.96}} & \multicolumn{1}{l|}{\textbf{0.29}} & 0.87 & \multicolumn{1}{l|}{\textbf{0.96}} & \multicolumn{1}{l|}{\textbf{0.29}} & 0.83 & \multicolumn{1}{l|}{0.96} & \multicolumn{1}{l|}{\textbf{0.29}} & 0.74 \\ \hline
LTI & \multicolumn{1}{l|}{0.54} & \multicolumn{1}{l|}{1.32} & \textbf{0.49} & \multicolumn{1}{l|}{0.88} & \multicolumn{1}{l|}{0.45} & \textbf{0.44} & \multicolumn{1}{l|}{{0.81}} & \multicolumn{1}{l|}{0.39} & \textbf{0.50} & \multicolumn{1}{l|}{\textbf{0.97}} & \multicolumn{1}{l|}{0.34} & \textbf{0.42} \\ \hline
\end{tabular}%
}

\end{table}



\section{Conclusion} \label{sec:lgipdaf_conclusion}

In this paper we have shown how to adapt the IPDA filter to connected, unimodular Lie groups. In our example, we showed that a constant-velocity target model on \SE{2} is significantly better able to track non-linear motion in dense clutter than the LTI-CV model. This is because the LTI-CV model expresses only linear motion and cannot predict non-linear motion such as curvy and circular trajectories. However, the \SE{2}-CV model expresses non-linear and linear motion, and thus it is able to predict both types of motion and track them well. We have also shown that the LG-IPDAF is capable of quickly rejecting and confirming tracks with high fidelity.

\appendices

\section{Proof of Lemma~\ref{lem:lgipdaf_belief_measurement}: Prediction Step}
\label{app:lgipdaf_proof_prediction}
\begin{proof}
According to the theorem of total probability, the probability of the track's current state conditioned on the previous measurements and the track representing the target is
\begin{align}\label{eq:lgipdaf_belief_pdf}
p\left(x_{k} \mid \epsilon_{k^-}, Z_{0:k^-} \right)  &=\int_{\tilde{x}_{k^{-} \mid k^{-}}}  p\left(x_{k }\mid x_{k^{-}}, \epsilon_{k^-} , Z_{0:k^-} \right) \notag\\
&p\left(x_{k^{-}} \mid \epsilon_{k^-}, Z_{0:k^-} \right)d \tilde{x}_{k^{-} \mid k^{-}}, \end{align}
where we integrate over the error state instead of the state by recalling that $\tilde{x}_{k^- \mid k^-}= \LogI{\G{x}}{ \hat{x}_{k^- \mid k^-}^{-1} x_{k^-} }$, and using the relation in equation~\eqref{eq:prob_x_def}.

Using the definitions of $p\left(x_{k }\mid x_{k^{-}}, \epsilon_{k^-} , Z_{0:k^-} \right)$ and $p\left(x_{k^{-}} \mid \epsilon_{k^-}, Z_{0:k^-} \right)$ in Equations~\eqref{eq:lgipdaf_state_previous} and~\eqref{eq:lgipdaf_approx_state_propagation}, Equation~\eqref{eq:lgipdaf_belief_pdf} becomes 
\begin{equation}\label{eq:lgipdaf_prob_with_L}
p\left(x_{k} \mid \epsilon_{k^-}, Z_{0:k^-} \right)\approx\eta\int_{\tilde{x}_{k^{-} \mid  k^{-}}}\exp\left(-L\right)d \tilde{x}_{k^{-} \mid  k^{-}},    
\end{equation}
where
\begin{align*}
L&=\frac{1}{2}\left(\tilde{x}_{k\mid k^{-}}-F_{\kt}\tilde{x}_{k^{-} \mid  k^{-}}\right)^{\top}\bar{Q}_{\kt}^{-1}\left(\tilde{x}_{k\mid k^{-}}-F_{\kt}\tilde{x}_{k^{-} \mid  k^{-}}\right)\\
&+\frac{1}{2}\tilde{x}_{k^{-} \mid  k^{-}}^{\top}P_{k^{-} \mid  k^{-}}^{-1}\tilde{x}_{k^{-} \mid  k^{-}},    
\end{align*}
and where $\bar{Q}_{\kt}$ is defined in equation~\eqref{eq:lgipdaf_q_bar}.  After some algebra we get
\begin{multline*}
    L= \frac{1}{2} \Big\{
    \xtkmkm^\top (\Pkmkm^{-1} + F_\kt^\top \bar{Q}_\kt^{-1}F_\kt) \xtkmkm
    \\
    + \xtkkm^\top \bar{Q}_{\kt}^{-1} \xtkkm 
    - \xtkkm \bar{Q}_\kt^{-1}F_\kt\xtkmkm 
    \\
    - \xtkmkm^\top F_\kt^\top \bar{Q}_\kt^{-1} \xtkkm
    \Big\} 
\end{multline*}
Defining $S \defeq \Pkmkm^{-1} + F_\kt^\top \bar{Q}_\kt^{-1}F_\kt$ gives
\begin{multline*}
    L= \frac{1}{2} 
    (\xtkmkm-S^{-1}F_\kt^\top\bar{Q}_\kt^{-1}\xtkkm)^\top 
    \\
    S (\xtkmkm-S^{-1}F_\kt^\top\bar{Q}_\kt^{-1}\xtkkm)
    \\
    + \frac{1}{2}\xtkkm\big(\bar{Q}_\kt^{-1} - \bar{Q}_\kt^{-1} F_\kt S^{-1} F_\kt^\top \bar{Q}_\kt^{-1}\big)\xtkkm.
\end{multline*}
Using the matrix inversion lemma
\begin{multline*}
    \bar{Q}_\kt^{-1} - \bar{Q}_\kt^{-1} F_\kt (F_\kt^\top \bar{Q}_\kt^{-1} F_\kt + \Pkmkm^{-1})^{-1}F_\kt^\top \bar{Q}_\kt^{-1} 
    \\
    = (F_\kt \Pkmkm F_\kt^\top + \bar{Q}_\kt)^{-1}
\end{multline*}
gives
\begin{multline*}
    L= L_1(\xtkmkm, \xtkkm)
    \\
    + \frac{1}{2}\xtkkm(F_\kt \Pkmkm F_\kt^\top + \bar{Q}_\kt)^{-1}\xtkkm,
\end{multline*}
and therefore
\begin{align*}
    p\left(x_{k}\mid \epsilon_{k^-}, Z_{0:k^-} \right) 
     =  \eta \exp{\left(- L_2\left(\xtkkm\right)\right)},
\end{align*}
where the integral of $\exp{\left(-L_1\right)}$ has been absorbed into the normalizing coefficient $\eta$.

Let 
\[
\Pkkm = F_\kt \Pkmkm F_\kt^\top + \bar{Q}_\kt
\]
and recall that $\tilde{x}_{k \mid k^{-}} = \LogI{\G{x}}{\hat{x}_{k\mid k^{-}}^{-1}  x_{k} }$, where from Lemma~\ref{lem:lgipdaf_prob_prop_state} we have that $\hat{x}_{k\mid k^{-}}=f\left(\hat{x}_{k^{-} \mid  k^{-}},0,t_{\kt} \right)$.
Therefore, the Lemma follows by noting that $p\left(x_{k} \mid \epsilon_{k^-}, Z_{0:k^-} \right) \approx \mathcal{N}\left(\hat{x}_{k\mid k^{-}}, P_{k \mid k^{-}} \right)$.
\end{proof}
\section{Proof of Lemma~\ref{lem:lgipdaf_association_event}: Association Events}\label{app:lgipdaf_association_event}
\begin{proof}
The probability of an association event conditioned on the measurements $Z_{0:k}$ and that the track represents the target is inferred using the location of the validated measurements with respect to the track's estimated measurement $\hat{z}_{k} = h\left(\hat{x}_{k \mid k^-},0 \right)$ and the number of validated measurements. The basic idea is that the closer a validated measurement is to the estimated measurement relative to how close the other validated measurements are to the estimated measurement, the more likely the validated measurement is the target-originated measurement, and therefore the more likely its respective association event it. 

The probability of an association event conditioned on the measurements and the track representing the target is
\begin{equation}\label{eq:lgipdaf_prob_assoc_event1}
    p\left(\theta_{k,j} \mid \epsilon_k, Z_{0:k} \right) = p\left(\theta_{k,j} \mid \epsilon_k, Z_{k}, Z_{0:k^-}, m_k \right),
\end{equation}
where we have explicitly written the inference on the number of validated measurements $m_k$.
Using Bayes' rule, the probability in equation \eqref{eq:lgipdaf_prob_assoc_event1} is
\begin{align}\label{eq:lgipdaf_prob_assoc_event2}
    &p\left(\theta_{k,j}\mid \epsilon_k , Z_{0:k}\right) = \notag \\
    &\frac{p\left(Z_k \mid \theta_{k,j}, m_k, \epsilon_k, Z_{0:k^-} \right)p\left(\theta_{k,j} \mid m_k, \epsilon_k, Z_{0:k^-}\right)}{p\left(Z_{k} \mid m_k, \epsilon_k, Z_{0:k^-} \right)},
\end{align}
where 
\begin{align}\label{eq:lgipdaf_prob_assoc_event3}
    &p\left(Z_{k} \mid m_k, \epsilon_k, Z_{0:k^-} \right) = \notag \\ 
    &\sum_{j=0}^{m_k}\left( p\left(Z_k \mid \theta_{k,j}, m_k, \epsilon_k, Z_{0:k^-} \right) \right. \left. p\left(\theta_{k,j} \mid m_k, \epsilon_k, Z_{0:k^-} \right) \right).
\end{align}
Since the validated measurements are independent, the joint density of the validated measurements is
\begin{equation}\label{eq:lgipdaf_meas_prod}
p\left(Z_{k}\mid \theta_{k,j},m_{k}, \epsilon_k, Z_{0:k^-}\right)=\prod_{\ell=1}^{m_{k}}p\left(z_{k,\ell}\mid \theta_{k,j}, \epsilon_k, Z_{0:k^-}\right).
\end{equation}
Since the false measurements are assumed uniformly
distributed in the validation region with volume $\mathcal{V}_k$, the probability of a false measurement in the validation region is $\mathcal{V}_k^{-1}$. Also, since the target originated measurement is validated with probability $P_G$, the probability of a target originated measurement being validated is $P_{G}^{-1}p\left(z_{k\ell}\mid \psi, \epsilon_k, Z_{0:k^-} \right)$, which is defined in equations \eqref{eq:lgipdaf_probability_y} \eqref{eq:lgipdaf_gate_prob} and \eqref{eq:lgipdaf_prob_meas_val}. Therefore, the probability of each measurement conditioned on the respective association event is
\begin{multline}\label{eq:lgipdaf_meas_prob_cond}
p\left(z_{k,\ell}\mid \theta_{k,j}, \epsilon_k, Z_{0:k^-} \right)
\\
=
\begin{cases}
\mathcal{V}_{k}^{-1} & \text{if }j\neq \ell\\
P_{G}^{-1}p\left(z_{k,\ell}\mid \psi, \epsilon_k, Z_{0:k^-}\right) & \text{if }j = \ell
\end{cases}.
\end{multline}

Let $m^{-}_{k} \triangleq m_k -1 $, and $\mathbb{M}=\left\{1,\ldots,m_k\right\}$. Substituting equation \eqref{eq:lgipdaf_meas_prob_cond} into equation  \eqref{eq:lgipdaf_meas_prod} yields the joint probability of the measurements
\begin{multline} \label{eq:lgipdaf_joint_dist_meas}
 p\left(Z_{k}\mid \theta_{k,j},m_{k}, \epsilon_k, Z_{0:k^-}\right) 
 \\
=\begin{cases}
\mathcal{V}_{k}^{-m^-_{k}}P_{G}^{-1}p\left(z_{k,j}\mid \psi, \epsilon_k, Z_{0:k^-} \right) & \text{if } j \in \mathbb{M}\\
\mathcal{V}_{k}^{-m_{k}} & \text{if } j=0
\end{cases}.
\end{multline}

We now proceed to calculate $p\left(\theta_{k,j} \mid m_k, \epsilon_k, Z_{0:k^-}\right)$. Let $\phi$ denote the number of false measurements. Under the assumption that there is
at most one target originated measurement, there are two possibilities
for the number of false measurements, $\phi = {m_k}$, denoted $\phi_{m_k}$, or $\phi = {m_k^-}$ denoted $\phi_{m_k^-}$.
Therefore, the {\em a~priori} probability of an association event conditioned on the
number of measurements and previous measurements is 
\begin{align}\label{eq:lgipdaf_event_cond_prob}
& p\left(\theta_{k,j}\mid m_{k}, \epsilon_k, Z_{0:k^-}\right) = \notag \\
& p\left(\theta_{k,j}\mid \phi_{m_k^-},m_{k},\epsilon_k,Z_{0:k^-}\right)p\left(\phi_{m_k^-}\mid m_{k},\epsilon_k,Z_{0:k^-}\right) \notag \\
 & +p\left(\theta_{k,j}\mid\phi_{m_k},m_{k}, \epsilon_k,Z_{0:k^-}\right)p\left(\phi_{m_k}\mid m_{k},\epsilon_k,Z_{0:k^-}\right)\notag\\
 & =\begin{cases}
\left(\frac{1}{m_{k}}\right)p\left(\phi_{m_k^-}\mid m_{k},\epsilon_k,Z_{0:k^-}\right) & j=1,\ldots,m_{k}\\
p\left(\phi_{m_k}\mid m_{k},\epsilon_k,Z_{0:k^-}\right) & j=0,
\end{cases}
\end{align}
where the probability $p\left(\theta_{k,j}\mid \phi_{m_k^-},m_{k},\epsilon_k,Z_{0:k^-}\right)= \frac{1}{m_k} $ when $j>0$ since each association event, $\theta_{k,j>0}$, is just as likely to be true with the specified conditions, and \\ $p\left(\theta_{k,j}\mid \phi_{m_k^-},m_{k},\epsilon_k,Z_{0:k^-}\right)= 0 $ when $j=0$ since not all of the validated measurements can be false according to the condition $\phi_{m_k^-}$.

Using Bayes' formula, the conditional probabilities of the number of false measurements are
\begin{align}
 &p\left(\phi_{m_k^-}\mid m_{k},\epsilon_k, Z_{0:k^-}\right) \notag \\
 & =\frac{p\left(m_{k}\mid\phi_{m_k^-}, \epsilon_k, Z_{0:k^-}\right) p\left(\phi_{m_k^-}\mid \epsilon_k, Z_{0:k^-}\right)}{p\left(m_{k}\mid \epsilon_k, Z_{0:k^-}\right)}\notag\\
 & =\frac{P_{G}P_{D}\mu_{F}\left(m_{k}^{-}\right)}{p\left(m_{k}\mid \epsilon_k, Z_{0:k^-}\right)}\label{eq:lgipdaf_num_false_meas_cond_prob1}
\end{align}
and 
\begin{align}
& p\left(\phi_{m_k}\mid m_{k},\epsilon_k, Z_{0:k^-}\right) \notag \\
&=\frac{p\left(m_{k}\mid\phi_{m_k}, \epsilon_k, Z_{0:k^-}\right)p\left(\phi_{m_k}\mid Z_{0:k^-}\right)}{p\left(m_{k}\mid \epsilon_k, Z_{0:k^-}\right)}\notag\\
 & =\frac{\left(1-P_{G}P_{D}\right)\mu_{F}\left(m_{k}\right)}{p\left(m_{k}\mid \epsilon_k, Z_{0:k^-}\right)},
 \label{eq:lgipdaf_num_false_meas_cond_prob2}
\end{align}
where $\mu_{F}$ is the probability density function of the number
of false measurements, and $p\left(m_{k}\mid\phi_{m_k^-}, \epsilon_k, Z_{0:k^-}\right)=P_G P_D$ since $p\left(m_{k}\mid\phi_{m_k^-}, \epsilon_k, Z_{0:k^-}\right)$ is the probability that the target is detected and the target originated measurement is inside the validation region.

According to the theorem of total probability
\begin{align}
& p\left(m_{k}\mid \epsilon_k, Z_{0:k^-}\right) \notag \\
 =& p\left(m_{k}\mid\phi_{m_k^-}, \epsilon_k, Z_{0:k^-}\right)p\left(\phi_{m_k^-}\mid \epsilon_k,  Z_{0:k^-}\right) \notag\\
&+p\left(m_{k} \mid \phi_{m_k}, \epsilon_k, Z_{0:k^-}\right)p\left(\phi_{m_k}\mid \epsilon_k, Z_{0:k^-}\right)\notag \\
= & P_{G}P_{D}\mu_{F}\left(m_{k}^{-}\right)+\left(1-P_{G}P_{D}\right)\mu_{F}\left(m_{k}\right).\label{eq:lgipdaf_num_meas_prob}
\end{align}
Substituting equations \eqref{eq:lgipdaf_num_false_meas_cond_prob1}, \eqref{eq:lgipdaf_num_false_meas_cond_prob2}, and \eqref{eq:lgipdaf_num_meas_prob} into equation \eqref{eq:lgipdaf_event_cond_prob} yields
\begin{multline}\label{eq:lgipdaf_event_cond_prob_exist_final}
p\left(\theta_{k,j}\mid m_{k}, \epsilon_k, Z_{0:k^-}\right) 
\\
=\begin{cases}
\frac{\frac{1}{m_{k}}P_{D}P_{G}\mu_{F}\left(m_{k}^{-}\right)}{P_{G}P_{D}\mu_{F}\left(m_{k}^{-}\right)+\left(1-P_{G}P_{D}\right)\mu_{F}\left(m_{k}\right)} & j\in \mathbb{M}\\
\frac{\left(1-P_{G}P_{D}\right)\mu_{F}\left(m_{k}\right)}{P_{G}P_{D}\mu_{F}\left(m_{k}^{-}\right)+\left(1-P_{G}P_{D}\right)\mu_{F}\left(m_{k}\right)} & j=0
\end{cases}.
\end{multline}
Substituting equations \eqref{eq:lgipdaf_joint_dist_meas} and \eqref{eq:lgipdaf_event_cond_prob}, into equation \eqref{eq:lgipdaf_prob_assoc_event3} yields 
\begin{multline}
p\left(Z_{k}\mid m_k, \epsilon_k, Z_{0:k^-} \right) 
\\
=\frac{\frac{\mathcal{V}_{k}^{-m_{k}^{-}}}{m_{k}}P_{D}\mu_{F}\left(m_{k}^{-}\right)\sum_{\ell=1}^{m_{k}}p\left(z_{k,\ell}\mid\psi, \epsilon_k, Z_{0:k^-}\right)}{P_{D}P_{G}\mu_{F}\left(m_{k}^{-}\right)+\left(1-P_{D}P_{G}\right)\mu_{F}\left(m_{k}\right)}
\\
  + \frac{\mathcal{V}_{k}^{-m_{k}}\left(1-P_{D}P_{G}\right)\mu_{F}\left(m_{k}\right)}{P_{D}P_{G}\mu_{F}\left(m_{k}^{-}\right)+\left(1-P_{D}P_{G}\right)\mu_{F}\left(m_{k}\right)}.\label{eq:lgipdaf_prob_zk_e}
\end{multline}
Substituting equations \eqref{eq:lgipdaf_joint_dist_meas}, \eqref{eq:lgipdaf_event_cond_prob} and \eqref{eq:lgipdaf_prob_zk_e} into \eqref{eq:lgipdaf_prob_assoc_event2} yields
{\small
\begin{align}
& p\left(\theta_{k,j} \mid \epsilon_k, Z_{0:k} \right)= &\notag \\ 
&\begin{cases}
\frac{P_{D}p\left(z_{k,j}\mid \psi, \epsilon_k , Z_{0:k^-}\right)}{P_{D}\sum_{\ell=1}^{m_{k}}p\left(z_{k,j}\mid \psi, \epsilon_k , Z_{0:k^-}\right)+m_{k}\mathcal{V}_{k}^{-1}\left(1-P_{D}P_{G}\right)\frac{\mu_{F}\left(m_{k}\right)}{\mu_{F}\left(m_{k}^{-}\right)}} & j\in\mathbb{M}\\
\frac{m_{k}\mathcal{V}_{k}^{-1}\left(1-P_{D}P_{G}\right)\frac{\mu_{F}\left(m_{k}\right)}{\mu_{F}\left(m_{k}^{-}\right)}}{P_{D}\sum_{\ell=1}^{m_{k}}p\left(z_{k,j}\mid \psi, \epsilon_k , Z_{0:k^-}\right)+m_{k}\mathcal{V}_{k}^{-1}\left(1-P_{D}P_{G}\right)\frac{\mu_{F}\left(m_{k}\right)}{\mu_{F}\left(m_{k}^{-}\right)}} & j=0
\end{cases}\label{eq:lgipdaf_prob_assoc_event4}
\end{align}
}
Setting the probability density function of false measurements $\mu_F$ in equation \eqref{eq:lgipdaf_prob_assoc_event4} to a Poisson density function defined in equation \eqref{eq:lgipdaf_poisson_distribution} yields equation \eqref{eq:lgipdaf_association_probabilities}.
\end{proof}

\section{Proof of Lemma~\ref{lem:lgipdaf_single_update}: Split Track Update}

We begin with the following Lemma.
\begin{lemma}\label{lem:lgipdaf_prob_z_cond_x}
Suppose that Assumptions~\ref{amp:lgipdaf_system_model} and \ref{amp:lgipdaf_past_info} hold, then the probability of measurement $z_k$ conditioned on it being target-originated and conditioned on the current state, is given by
\begin{multline}\label{eq:lgipdaf_prob_z_cond_x}
    p\left(z_{k} \mid \psi, x_{k}, \epsilon_{k}, Z_{0:k^-} \right) \approx
    \\  
    \eta \exp \left( - \frac{1}{2} \left( \tilde{z}_k - H_k \tilde{x}_{k \mid k^-} \right)^\top \bar{R}_k \left( \tilde{z}_k - H_k \tilde{x}_{k \mid k^-} \right) \right),
\end{multline}
where
\begin{align}
    \tilde{z}_k & = \LogI{\G{s}}{ \hat{z}_k ^{-1} z_k} \\
    \hat{z}_k & = h\left( \hat{x}_{k \mid k^-}, 0 \right) \\
    \bar{R}_k & = V_k R V_k ^\top, \\
    \xtkkm & = \LogI{\G{s}}{ \hat{x}_{k\mid k^-}^{-1} x_k},
\end{align}
$R$ is the measurement noise covariance, and the Jacobians $H_k$ and $V_k$ are given in equation~\eqref{eq:sm_system_jacobians} and evaluated at the point $\zeta_{k_k} = \left( \hat{x}_{k \mid k^-}, 0 \right)$.
\end{lemma}
\par\noindent{\em Proof:}
Similar to the proof of Lemma~\ref{lem:lgipdaf_prob_prop_state}.
\vspace{0.5cm}

\par\noindent{\em Proof of Lemma~\ref{lem:lgipdaf_single_update}.}
\label{proof:lgipdaf_single_update}
Using Bayes rule, 
\begin{multline}\label{eq:lgipdaf_map_bayes}
p\left(x_{k,j} \mid \theta_{k,j}, \epsilon_k, Z_{0:k} \right)
\\
=\frac{p\left(Z_{k}|\theta_{k,j}, x_{k,j}, \epsilon_k, Z_{0:k^-} \right)p\left(x_{k,j} \mid \epsilon_k, Z_{0:k^-} \right)}{p\left(Z_{k}\mid \theta_{k,j}, \epsilon_k, Z_{0:k^-} \right)}.  
\end{multline}
We solve for the probability $p\left(x_{k,j} \mid \theta_{k,j}, \epsilon_k, Z_{0:k} \right)$ by using the maximum a posterior (MAP) optimization algorithm to find the value of $\tilde{x}_{k \mid k^-}$ and its corresponding error covariance $P_{k \mid k^-}$ that maximizes the right-hand-side of equation~\eqref{eq:lgipdaf_map_bayes}.

Since the MAP does not depend on $p\left(Z_{k}\mid \theta_{k,j}, \epsilon_k, Z_{0:k^-} \right)$, it can be absorbed into the normalizing coefficient simplifying the problem to 
\begin{align*}
&\underset{ \tilde{x}_{k \mid k^-}, P_{k \mid k^-}  }{\text{max}}\eta p\left(Z_{k}|\theta_{k,j}, x_{k,j}, \epsilon_k, Z_{0:k^-} \right)p\left(x_{k,j} \mid \epsilon_k, Z_{0:k^-} \right).    
\end{align*}
When $j=0$, none of the validated measurements are target originated, and thus, 
$p\left(Z_{k}|\theta_{k,j=0}, x_{k}, \epsilon_k, Z_{0 \mid k^-} \right)$ simplifies to $p\left(Z_{k}|\theta_{k,j=0}, Z_{0:k^-} \right)$ and no longer has any dependency on the track's state. This reduces the MAP optimization problem to 
\[
\underset{\tilde{x}_{k\mid k^-}, P_{k\mid k^-}}{\text{max}}\eta p\left( x_{k,j} \mid \epsilon_k, Z_{0:k^-} \right),
\]
and hence, the solution is $p\left(x_{k,j=0} \mid \theta_{k,j=0}, \epsilon_k, Z_{0:k} \right) = p\left( x_{k} \mid \epsilon_k, Z_{0:k^-} \right)$.

When $j>0$, only the measurement $z_{k,j}$ has any dependency and influence on the track's state since all other measurements are false. Using this fact with the relation in Equation~\eqref{eq:lgipdaf_meas_prob_cond}, the optimization problem simplifies to
\begin{multline*}
\underset{\tilde{x}_{k\mid k^-}, P_{k\mid k^-}}{\text{max}} \eta P_{G}^{-1}p\left(z_{k,j}|\psi, x_{k,j}, \epsilon_k, Z_{0:k^-} \right) \\
\cdot p\left( x_{k,j} \mid \epsilon_k, Z_{0:k^-} \right).
\end{multline*}

Multiplying together and combining the exponents of the probabilities $p\left(z_{k,j} \mid \psi, x_{k}, \epsilon_{k^-}, Z_{0:k^-} \right)$ and $p\left(x_{k} \mid \epsilon_{k^-}, Z_{0:k^-} \right)$, defined in equations~\eqref{eq:lgipdaf_prob_z_cond_x} and \eqref{eq:lgipdaf_prob_x_curr_cond_prev_meas}, simplifies the MAP optimization problem to
\[
\underset{\tilde{x}_{k\mid k^-}, P_{k\mid k^-}}{\text{max}} \eta P^{-1}_G \exp \left(-L\right),
\]
where 
\begin{align*}
L=&\frac{1}{2}\left(\nu_{k,j}-H_{k}\tilde{x}_{k \mid k^{-}}\right)^{\top}\bar{R}_{k}^{-1}\left(\nu_{k,j}-H_{k}\tilde{x}_{k \mid k^{-}}\right) \\
&+\frac{1}{2}\tilde{x}_{k \mid k^{-}}^{\top}P_{k \mid k^{-}}^{-1}\tilde{x}_{k \mid k^{-}} \\
\nu_{k,j}  = & \LogI{\G{s}}{ \hat{z}_k ^{-1} z_k} \\
\hat{z}_k  = & h\left( \hat{x}_{k \mid k^-}, 0 \right) \\
\bar{R}_k  = & V_k R V_k ^\top,
\end{align*}
and where $\tilde{x}_{k \mid k^-}  = \LogI{\G{x}}{ \hat{x}_{k \mid k^-}^{-1} x_k}$, $h$ is the observation function defined in equation~\eqref{eq:sm_system_model}, and $H_k$ and $V_k$ are the Jacobians of the observation function defined in equation~\eqref{eq:sm_system_jacobians} and evaluated at $\zeta_{h_k} = \left( \hat{x}_{k \mid k^-} ,0 \right)$.

The MAP is solved by finding the value of $\tilde{x}_{k\mid k^{-}}$ that minimizes $L$. Since $L$ is quadratic in $\tilde{x}_{k\mid k^{-}}$, the value of $\tilde{x}_{k\mid k^{-}}$ that minimizes $L$ is found by taking the first derivative of $L$ with respect to $\tilde{x}_{k\mid k^{-}}$, setting it to zero and solving for $\tilde{x}_{k\mid k^{-}}$. This value becomes the new error state mean $\mu^{-}_{k\mid k}$. The corresponding covariance is found by taking the second derivative of $L$ with respect to $\tilde{x}_{k\mid k^{-}}$ and setting this value to the new covariance $P^{c^-}_{k\mid k}$.

Taking the first and second partial derivatives of $L$ with respect to $\tilde{x}_{k \mid k^{-}}$
yields
\begin{align*}
\frac{\partial L}{\partial\tilde{x}_{k \mid k^{-}}} & =-\left(\nu_{k,j}-H_{k}\tilde{x}_{k \mid k^{-}}\right)^{\top}\bar{R}_{k}^{-1}H_{k}+\tilde{x}_{k \mid k^{-}}^{\top}P_{k \mid k^{-}}^{-1}\\
\frac{\partial^{2}L}{\partial\tilde{x}_{k \mid k^{-}}^{2}} & =H_{k}^{\top}\bar{R}^{-1}_{k}H_{k}+P_{k \mid k^{-}}^{-1} = \left(P^{c^-}_{k|k}\right)^{-1}.
\end{align*}
Setting the first derivative to zero, solving for $\tilde{x}_{k \mid k^{-}}$ and setting this value to $\mu^{-}_{k\mid k}$ gives 
\begin{align*}
\mu^{-}_{k|k} & =\left(H_{k}^{\top}\bar{R}^{-1}_{k}H_{k}+P_{k \mid k^{-}}^{-1} \right)^{-1}H_{k}^{\top}\bar{R}_{k}^{-1}\nu_{k,j}.
\end{align*}
With algebraic manipulation the updated error covariance and error state mean are
\begin{align*}
    P^{c^-}_{k|k} &= \left(I-K_{k} H_{k}\right)P_{k \mid k^{-}} \\
    \mu^{-}_{k|k,j} & =K_{k}\nu_{k,j},
\end{align*}
where the Kalman gain $K_k$ and innovation term $\nu_{k,j}$ are 
\begin{align*}
    K_k & = P_{k \mid k^-} H_{k}^\top S_k \\
    \nu_{k,j}&=\LogI{\G{s}}{ \hat{z}_{k}^{-1} z_{k,j} } \\
    \hat{z}_k & = h\left( \hat{x}_{k \mid k^- } ,0 \right) \\
    S_{k} & =V_{k}RV_{k}^{\top}+H_{k}P_{k|k^-}H_{k}^{\top}. \\
\end{align*}

Since the error state's mean is no longer zero, the error state no longer has a concentrated Gaussian distribution. In order to reset the mean of the error state to zero, we add $\mu^{-}_{k|k,j}$ onto the state estimate $\hat{x}_{k \mid k^{-}}$ and adjust the covariance of the error state. In particular, let the error state after update but before being reset be $\tilde{x}^{-}_{k|k,j} = \mu^{-}_{k|k,j} + a_{k|k}$ where $a_{k|k} \sim \mathcal{N}\left(0, P^{c^-}_{k|k} \right)$.  Then under the assumption that $a_{k|k}$ is small and using the property of the right Jacobian defined in equation~\eqref{eq:Jr_property} add $\mu^{-}_{k|k,j}$ onto $\hat{x}_{k \mid k^{-}}$ to get
\begin{subequations}
\begin{align}
x_{k|k} & = \hat{x}_{k \mid k^{-}}\ExpI{\G{x}}{ \mu^{-}_{k|k,j} + a_{k|k} } \\
 & = \underbrace{\hat{x}_{k \mid k^{-}}\ExpI{\G{x}}{ \mu^{-}_{k|k,j} }}_{\hat{x}_{k\mid k,j}}\ExpI{\G{x}}{ \underbrace{\jr{\G{x}}{\mu^{-}_{k|k,j}} a_{k|k}}_{\tilde{x}_{k\mid k,j}} }, 
\end{align}
\end{subequations}
where
\begin{equation} \label{eq:lgipdaf_indirect_state_update}
    \hat{x}_{k|k,j} = \hat{x}_{k \mid k^{-}}\ExpI{\G{x}}{  \mu^{-}_{k|k,j} }
\end{equation}
is the updated state estimate, and $\tilde{x}_{k|k,j} =  \jr{\G{x}}{\mu^{-}_{k|k,j}} a_{k|k}$ is the updated error state after reset. Equation~\eqref{eq:lgipdaf_indirect_state_update} can be thought of as moving from $\hat{x}_{k \mid k^{-}}$ to $\hat{x}_{k|k,j}$ along the geodesic defined by the tangent vector  $\mu^{-}_{k|k,j}$ as depicted in Fig.~\ref{fig:lgipdaf_geodesic_add}. 

The error covariance of the error state $\tilde{x}_{k|k,j}$ is 
\begin{align*}
    \text{cov}\left(\tilde{x}_{k|k,j}\right) & = \text{cov} \left( \jr{\G{x}}{\mu^{-}_{k|k,j}} a_{k|k} \right) \\
    & = \jr{\G{x}}{\mu^{-}_{k|k,j}} \text{cov} \left( a_{k|k} \right) \jr{\G{x}}{\mu^{-}_{k|k,j}}^\top  \\
    & = \jr{\G{x}}{\mu^{-}_{k|k,j}}   P^{c^-}_{k|k} \jr{\G{x}}{\mu^{-}_{k|k,j}}^\top \\
    & = P^c_{k|k}.
\end{align*}
Therefore, $ \tilde{x}_{k|k,j}\sim \mathcal{N}\left(\mu_{k|k,j}=0, P^c_{k|k,j} \right) $.

\begin{figure}[htb]
\centering
    \includegraphics[width=0.9\linewidth]{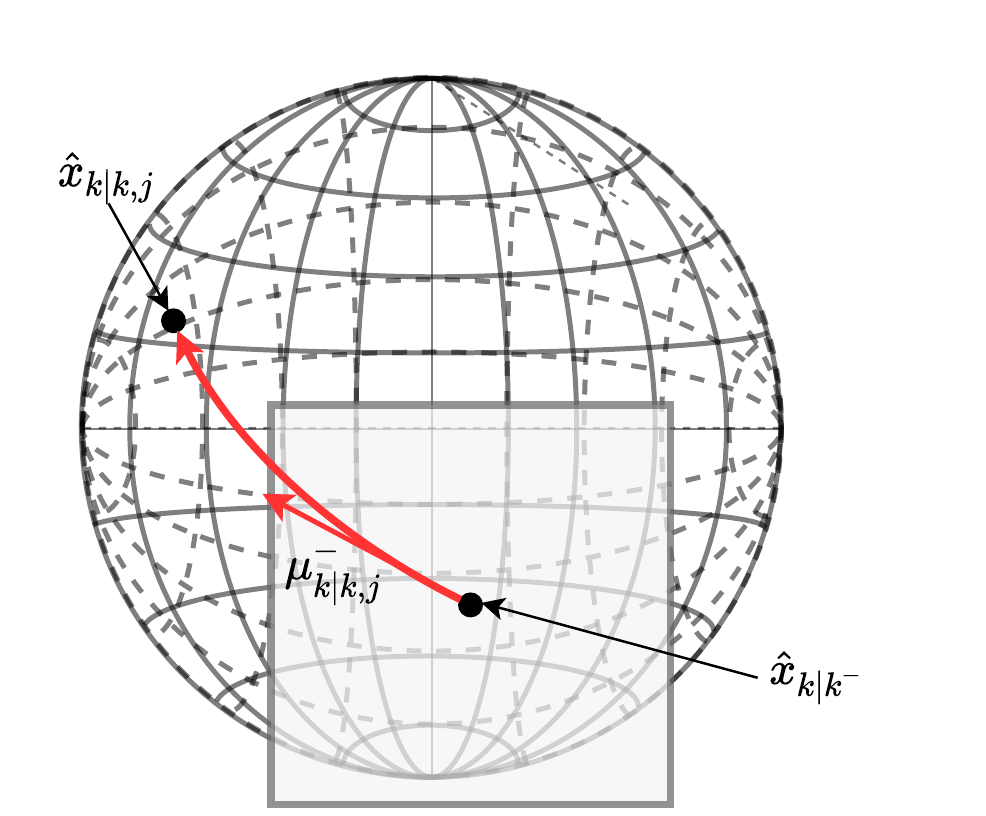}
    \caption{A depiction of the state estimate update conditioned on $\theta_{k,j}$ by using $\mu^{-}_{k|k,j}$ to form a geodesic from $\hat{x}_{k \mid k^{-}}$ to $\hat{x}_{k|k,j}$.}
     \label{fig:lgipdaf_geodesic_add}
\end{figure}

\section{Proof of Lemma~\ref{lem:lgipdaf_track_likelihood}: Track Likelihood}\label{proof:lgipdaf_track_likelihood}
The probability of the track likelihood conditioned on the measurements $Z_{0:k}$ are inferred using the location of the validated measurements with respect to the track's estimated measurement $\hat{z}_{k} = h\left(\hat{x}_{k \mid k^-},0 \right)$ and the number of validated measurements. The basic idea is that the more likely at least one of the validated measurements is target originated compared to the likelihood that none of the validated measurements are target originated, the more likely the track represents a target. Also, if the number of validated measurements is more than the expected number of false measurements inside the validation region, then the more likely one of the validated measurements originated from the target and the track represents the target.

We write the track likelihood conditioned on the measurements as
\begin{equation}\label{eq:lgipdaf_prob_track_likelihood1}
    p\left(\epsilon_{k} \mid Z_{0:k} \right) = p\left(\epsilon_k \mid m_k,Z_{k}, Z_{0:k^-} \right),
\end{equation}
where we have separated $Z_k$ and $Z_{0:k^-}$ from $Z_{0:k}$ and explicitly written the inference on the number of validated measurements $m_k$.
Using Bayes' rule, equation \eqref{eq:lgipdaf_prob_track_likelihood1} can be written as
\begin{equation}\label{eq:lgipdaf_prob_track_likelihood2}
    p\left(\epsilon_{k}, \mid Z_{0:k} \right) = \frac{p\left(Z_k \mid \epsilon_k, m_k, Z_{0:k^-} \right)p\left(\epsilon_{k} \mid m_k, Z_{0:k^-} \right)}{p\left(Z_k \mid m_k, Z_{0:k^-} \right)},
\end{equation}
where
\begin{align}\label{eq:lgipdaf_prob_track_likelihood3}
    & p\left(Z_k \mid m_k, Z_{0:k^-} \right)  \notag \\
    &= p\left(Z_{k} \mid \epsilon_k, m_k, Z_{0:k^-} \right)p\left(\epsilon_{k} \mid m_k, Z_{0:k^-}\right) \notag \\
    &+ p\left(Z_{k} \mid \epsilon_k=F, m_k, Z_{0:k^-} \right)p\left(\epsilon_{k}=F \mid m_k, Z_{0:k^-}\right)
\end{align}
and $\epsilon_k =F$ denotes that the track does not represent the target.

The probability $p\left(Z_{k} \mid m_k,\epsilon_k, Z_{0:k^-} \right)$ is derived in Appendix~\ref{app:lgipdaf_association_event} and defined in equation \eqref{eq:lgipdaf_prob_zk_e}. Accordingly
\begin{align}\label{eq:lgipdaf_prob_zk_e_f}
    &p\left(Z_{k} \mid \epsilon_k=F, m_k, Z_{0:k^-} \right) \notag \\& = \prod^{m_k}_{j=1}  p\left(z_{k,j} \mid \epsilon_k=F, m_k, Z_{0:k^-} \right) \notag \\
    & = \prod^{m_k}_{j=1} \mathcal{V}^{-1}_k 
     = \mathcal{V}_k^{-m_k}
\end{align}
because under the condition that the track does not represent the target, all of the validated measurements must be false, and the validated false measurements are assumed to be independent and uniformly distributed in the validation region with volume $\mathcal{V}_k$.

Using Bayes' rule we get
\begin{align}
p\left(\epsilon_{k} \mid m_{k}, Z_{0:k^-}\right)
=\frac{p\left(m_{k}\mid \epsilon_{k},Z_{0:k^-}\right)p\left( \epsilon_{k} \mid , Z_{0:k^-}\right)}{p\left(m_{k}\mid Z_{0:k^-}\right)} \label{eq:lgipdaf_prob_ek_mk} 
\end{align}
and
\begin{multline}
p\left(\epsilon_{k}=F\mid m_{k}, Z_{0:k^-}\right) 
 \\
=\frac{p\left(m_{k}\mid \epsilon_{k}=F, Z_{0:k^-}\right)p\left(\epsilon_{k}=F \mid Z_{0:k^-}\right)}{p\left(m_{k}\mid Z_{0:k^-}\right)}.
\end{multline}
The probability $p\left(m_{k}\mid \epsilon_{k},Z_{0:k^-}\right)$ is derived in Appendix~\ref{app:lgipdaf_association_event} and defined in equation \eqref{eq:lgipdaf_num_meas_prob}. The probability $p\left(m_{k}\mid \epsilon_{k}=F, Z_{0:k^-}\right)$ is 
\begin{equation}\label{eq:lgipdaf_prob_mk_ef}
    p\left(m_{k}\mid \epsilon_{k}=F, Z_{0:k^-}\right) = \mu_F\left(m_k\right),
\end{equation}
since all of the measurements are false under the condition that the track does not represent a target.

Using the theorem of total probability with equations  \eqref{eq:lgipdaf_num_meas_prob} and \eqref{eq:lgipdaf_prob_mk_ef}, the probability of the number of measurements $m_k$ conditioned on the previous measurements is 
\begin{subequations}\label{eq:lgipdaf_prob_mk}
\begin{align}
    p\left(m_{k}\mid Z_{0:k^-}\right) 
    =& p\left(m_{k}\mid \epsilon_{k},Z_{0:k^-}\right)P_T \notag \\
    &+ p\left(m_{k}\mid \epsilon_{k}=F, Z_{0:k^-}\right)p\left(\epsilon_{k\mid k^-}=F\right) \\
    = &P_{G}P_{D}\mu_{F}\left(m_{k}^{-}\right)P_T \notag \\
    & +\left(1-P_D P_G p\left( \epsilon_{k\mid k^-}\right) \right)\mu_{F}\left(m_k\right),
\end{align}
\end{subequations}
where $P_T \defeq p\left(\epsilon_{k} \mid Z_{0:k^-} \right)$.
Substituting equations \eqref{eq:lgipdaf_num_meas_prob} and \eqref{eq:lgipdaf_prob_mk} into equation \eqref{eq:lgipdaf_prob_ek_mk} yields
\begin{align} \label{eq:lgipdaf_prom_ek_t_mk}
    &p\left(\epsilon_{k} \mid m_{k}, Z_{0:k^-}\right)  \notag \\ 
    &=\frac{ \left(P_{G}P_{D}\mu_{F}\left(m_{k}^{-}\right)+\left(1-P_{G}P_{D}\right)\mu_{F}\left(m_{k}\right)\right) P_T}{P_{G}P_{D}\mu_{F}\left(m_{k}^{-}\right)P_T
    +\left(1-P_D P_G P_T \right)\mu_{F}\left(m_k\right)}.
\end{align}

Using the fact that 
\[p\left(\epsilon_{k} = F \mid m_{k}, Z_{0:k^-}\right)= 1-p\left(\epsilon_{k} \mid m_{k}, Z_{0:k^-}\right),
\]
and substituting the equations \eqref{eq:lgipdaf_prom_ek_t_mk}, \eqref{eq:lgipdaf_prob_zk_e_f}, and \eqref{eq:lgipdaf_prob_zk_e} into equation \eqref{eq:lgipdaf_prob_track_likelihood3} yields the probability of the validated measurements conditioned on the previous measurements
\begin{align}\label{eq:lgipdaf_prob_zk_final}
    & p\left(Z_k \mid m_k, Z_{0:k^-} \right) \notag \\ & =  \frac{\frac{\mathcal{V}_{k}^{-m_{k}^{-}}}{m_{k}}P_{D}\mu_{F}\left(m_{k}^{-}\right)\sum_{\ell=1}^{m_{k}}p\left(z_{k,j}\mid \psi, \epsilon_k,  Z_{0:k^-}\right)P_T}{P_{G}P_{D}\mu_{F}\left(m_{k}^{-}\right)P_T+\left(1-P_{D}P_{G}P_T\right)\mu_{F}\left(m_{k}\right)} \notag \\
    &+\frac{\mathcal{V}_{k}^{-m_{k}}\left(1-P_{D}P_{G}P_T\right)\mu_{F}\left(m_{k}\right)}{P_{G}P_{D}\mu_{F}\left(m_{k}^{-}\right)P_T+\left(1-P_{D}P_{G}P_T\right)\mu_{F}\left(m_{k}\right)}.
\end{align}
Substituting in equations \eqref{eq:lgipdaf_prob_zk_e}, \eqref{eq:lgipdaf_prom_ek_t_mk}, and \eqref{eq:lgipdaf_prob_zk_final} into equation \eqref{eq:lgipdaf_prob_track_likelihood2} and setting the density of false measurements to the Poisson distribution defined in equation \eqref{eq:lgipdaf_poisson_distribution} yields the probability of the track likelihood conditioned on the previous measurements given in equations \eqref{eq:lgipdaf_existence_final} and \eqref{eq:lgipdaf_alpha_weights}.

\printbibliography

\end{document}